\documentclass[11pt]{article}

\usepackage{amssymb, amsmath, verbatim, amsthm,url, multirow,fullpage,mathtools, graphicx, subcaption, csvsimple, outlines}
\usepackage{longtable, rotating,makecell,array}
\usepackage[aligntableaux=top]{ytableau}
\usepackage[all]{xy}
\usepackage[files,2-]{pagesel}

\usepackage{soul}
\usepackage[colorinlistoftodos,textsize=footnotesize]{todonotes}

\newcommand{\note}{\todo[color=green!40]}

\setstcolor{red}

\newcommand{\id}{\mathbb{I}}
\newcommand{\cM}{\mathcal{M}}
\newcommand{\cI}{\mathcal{I}}
\newcommand{\cR}{\mathcal{R}}
\newcommand{\cL}{\mathcal{L}}
\newcommand{\cD}{\mathcal{D}}
\newcommand{\cU}{\mathcal{U}}

\newcommand{\R}{\mathbb{R}}
\newcommand{\C}{\mathbb{C}}
\newcommand{\bI}{\mathbb{I}}
\newcommand{\D}{\partial}
\DeclareMathOperator{\Tr}{Tr}

\def\ba #1\ea{\begin{align} #1 \end{align}}
\def\bas #1\eas{\begin{align*} #1 \end{align*}}
\def\bml #1\eml{\begin{multline} #1 \end{multline}}
\def\bmls #1\emls{\begin{multline*} #1 \end{multline*}}

\newtheorem{thm}{Theorem}[section]

\newtheorem{conj}[thm]{Conjecture}
\newtheorem{lem}[thm]{Lemma}

\newtheorem{prop}[thm]{Proposition}

\theoremstyle{remark}

\theoremstyle{definition}
\newtheorem{dfn}[thm]{Definition}

\title{Geometry and Generalization: Eigenvalues as predictors of where a network will fail to generalize}

\author{Susama Agarwala, Benjamin Dees, Andrew Gearhart, Corey Lowman}

\begin{document}

\begin{filecontents*}{figs_and_tabs/graphs_seed-0_epochs-300/exp-default/coef_df_Log_VF_combined2.csv}
dim,,,Small,,,,,Big,,
,Intercept,$\log\omega\cL$,test_ind,test_$\log\omega\cL$,test_slope,Intercept,$\log\omega\cI$,test_ind,test_$\log\omega\cI$,test_slope
2,-0.0144,-0.0897,0.017,0.0066,-0.083,-0.0341,0.067,0.021,-4E-04,0.067
3,-0.0251,-0.0482,0.04,-0.0059,-0.054,-0.0351,0.08,0.032,0.0225,0.103
4,-0.0244,-0.0335,0.05,-0.0205,-0.054,-0.0363,0.127,0.05,0.0092,0.136
5,-0.0319,0.0014,0.092,-0.0364,-0.035,-0.0362,0.147,0.078,0.0505,0.197
6,-0.0199,0.1632,0.08,-0.0102,0.153,-0.0252,0.234,0.074,0.0394,0.273
7,-0.0166,0.2333,0.087,-0.0128,0.221,-0.0233,0.295,0.082,0.0225,0.317
8,-0.0159,0.2641,0.073,-0.0087,0.255,-0.0187,0.312,0.061,0.0179,0.33
9,-0.0111,0.343,0.066,0.0238,0.367,-0.0138,0.393,0.064,0.0335,0.427
10,-0.0148,0.3665,0.073,0.0113,0.378,-0.0179,0.382,0.063,0.0114,0.393
11,-0.0124,0.389,0.063,0.0199,0.409,-0.0134,0.402,0.06,0.011,0.414
12,-0.0088,0.4204,0.037,0.0012,0.422,-0.0106,0.428,0.038,-0.0043,0.424
13,-0.0123,0.4223,0.058,0.0044,0.427,-0.0122,0.426,0.058,0.0056,0.431
14,-0.0075,0.4214,0.052,0.0218,0.443,-0.0106,0.426,0.046,0.0029,0.429
15,-0.0111,0.453,0.056,0.0109,0.464,-0.012,0.457,0.051,0.0049,0.462
16,-0.0106,0.4541,0.052,0.0162,0.47,-0.0123,0.458,0.049,0.0039,0.462
17,-0.0071,0.4947,0.033,0.005,0.5,-0.0075,0.495,0.033,0.0068,0.501
18,-0.0095,0.5159,0.055,0.013,0.529,-0.0099,0.515,0.053,0.0098,0.525
19,-0.0095,0.4933,0.053,0.0093,0.503,-0.0092,0.496,0.048,0.0044,0.501
20,-0.0094,0.4887,0.043,0.0069,0.496,-0.0094,0.492,0.04,0.0115,0.503
\end{filecontents*}
\csvautotabular[respect underscore=true]{figs_and_tabs/graphs_seed-0_epochs-300/exp-default/coef_df_Log_VF_combined2.csv} 
\setcounter{page}{0}
\newpage
\maketitle

\begin{abstract}
We study the deformation of the input space by a trained autoencoder via the Jacobians of the trained weight matrices. In doing so, we prove bounds for the mean squared errors for points in the input space, under assumptions regarding the orthogonality of the eigenvectors. We also show that the trace and the product of the eigenvalues of the Jacobian matrices is a good predictor of the MSE on test points. This is a dataset independent means of testing an autoencoder's ability to generalize on new input. Namely, no knowledge of the dataset on which the network was trained is needed, only the parameters of the trained model. 
\end{abstract}

\section{Introduction}
Most neural networks see three different sets of data: the training data, the test data, and the deployment data. The training and test data is known to the engineer building the network, here, generically called the input dataset. The training and test data are complementary parts of this input dataset, with the training data is drawn i.i.d from this population. The deployment data, however, is the data that the network encounters after it as been trained, tested and installed in a production environment. While the training and validation datasets are presumed to be representative of the data the network will encounter after it has been built, the distribution of the deployment data is rarely guaranteed. This is the problem of generalizability for neural networks. It is imperative for the trustworthiness of neural networks to be able to predict when the errors of an neural network will be large on the deployment data. 

One way of generating this trustworthiness has been by studying the structure of the input dataset. There is a large body of work studying the intrinsic dimension and geometry of the input dataset \cite{ID, moreID, UMAP, tSNE}, and using this information to make accurate measurements of how far a new data point is from the original dataset, thus predicting the network's performance on the new data. While manifold learning techniques, principal component analysis, etc. have been very successful at understanding the structure of the dataset, these tools do not consider the structure of the model learned by the network. 

However, how well a neural network performs on data points not in the training set depends not only on the shape of the dataset, but also on the model (including the network architecture) that the neural network has learned from said dataset, seen as a collection of weight matrices, biases and activation functions. In this paper, we study the geometric properties of the model learned by an autoencoder, taking into account variations of network architecture and training time.  This has the advantage of creating a data independent measure of network performance. That is, one can predict how well a neural network will perform on a given data point, with \emph{no knowledge} of the original dataset used for training and validation. This is useful, as in actuality, the deployer of a neural network only has access to a trained network, with little to no knowledge of the data on which it was trained, or how it was trained. 

In this paper, we study the local geometry of the model learned by an autoencoder, namely the eigenvalues of the Jacobian matrix of the trained model at points in the input space, to predict the model's performance on test points (on MNIST data\footnote{The MNIST dataset is a standard dataset for developing new machine learning techniques, comprising of 70000 hand written digits, roughly evenly distributed between 0 and 9.}). We find that the eigenvalues are a good predictor of the mean squared error (MSE), also known as reconstruction error, on data points the model has not trained on. Furthermore, we find that they predict higher MSE loss on test points than on the training points. In other words, the eigenvalues are a good predictor of generalizability on the training and validation dataset. This gives hope that on data points that are not drawn from the same distribution as the test and training data, the eigenvalues will give a good prediction on the MSE loss, i.e. that the eigenvalues will prove to be a good predictor of generalizability on the deployment data.

An autoencoder consists of a pair of neural networks working in tandem to learn a minimal number of features needed to reconstruct the data point. It is an extremely useful tool for working with unlabeled data, or unsupervised learning. It is hypothesized that the features thus extracted do a good enough job of summarizing the unlabeled data that they can be used as inputs for other tasks involving the dataset (see Section 5 of \cite{DLinNN} and Chapter 14 in \cite{Goodfellowbook}). One network of the autoencoder (the encoder) projects the data down to a lower dimensional space (a latent space) that contains only the key information, while the other network (the decoder) ``inverts" this process to recreate the original data points. The pair of networks are  trained by minimizing the average distance squared between the input point $x$ and its image under the autoencoder, i.e. the MSE. This is the only test for whether or not the second network is, indeed, a inverse of the first. There is no further measure of whether or higher level derivatives of the two maps behave as one would expect if they were truly inverses.

A fundamental assumption underlying the autoencoder is that the data lies noisily around a lower dimensional data manifold $M_\cD$. The goal of projecting down to a latent space is to capture the information of the data manifold, rather than attempt to study its embedding. Therefore, if the latent space is of at least the same dimension as $M_\cD$, one would expect the autoencoder to behave locally like a projection map onto the space $T_xM_\cD$. In other words, the autoencoder would map a data point, $x$, to itself and its Jacobian $J_\cI(x)$ would have $\dim(M_\cD)$ eigenvalues that are $1$ while the rest are $0$. Higher order derivatives of the map would depend on the curvature of $M_\cD$. We also study $J_\cI(x)$ to measure how far the trained autoencoder is from being exactly such a projection.

In this paper, we study the properties of $J_\cI$, and of a similarly defined Jacobian $J_\cL$. Namely, we show that the difference between the reconstruction error on a test point and a training point is bounded by the Frobenius-norm distance of $J_\cI$ to a projection matrix, as long as the nontrivial eigenvectors of $J_\cI$ are sufficiently close to being orthogonal (Theorem \ref{res:bound}). In Sections \ref{sec:forms} and  \ref{sec:evals} we show that the distribution of eigenvalues of $J_\cL$ and $J_\cI$ are close to what one would expect if the autoencoder were a projection. Finally, in Section \ref{sec:predictions} we show that the traces and the determinants of $J_\cL$ (or the trace and product of the non-zero eigenvalues of $J_\cI$) are predictors of the MSE on test points. In particular, we show that as these quantities increase, at higher latent dimension, the predicted increase in expected reconstruction loss is several percent higher if the point is a test point, rather than a training point. This leads us to the main conclusion of this paper, that these quantities, derived from the structure of the trained model itself, rather than the structure of the data, gives predictive power on whether or not a network will generalize.

\subsection{Related work\label{sec:litreview}}

There is a growing body of work that studies the Jacobians of trained neural networks to understand the properties and structure of the data. 

One prominent example is the work done by Hauberg et al., who have a rigorous research program looking into the geometric underpinnings of machine learning, autoencoders, variational autoencoders and their latent spaces \cite{Hauberg:onlyBayes, Hauberg:enrichedlatent, Hauberg:pathlength, Hauberg:latentspaceoddity}. However, much of this work focuses on calculating distances between various data points. The approach taken in this paper is different. We are not interested in calculating pointwise distances in this paper. Rather, we are interested in studying neural networks in a pairwise fashion, in specific situations where they are supposed to be identical (or in this case, inverses of each other.) We propose a means of detecting when the prescribed relationship of equality or inversion does not hold, with the eventual aim of understanding how this may impact the system's performance beyond the input dataset, i.e. on the deployment data.

Another area where the Jacobians and their Frobenius norms appear is in the field of adversarial learning \cite{spectralnorm1, spectralnorm2, Parsevalnorm, Frobenius}. Here the authors define a regularization method that consists of bounding the largest singular value of the appropriate matrix. This regularization method leads to neural networks that generalize better. Their findings are consistent with our findings that as the trace of $J_\cI$ increases, so does the MSE on test points (i.e., the network generalizes poorly). In contrast, we do not propose a regularization method. The goal of this paper is not to train a more stable network, but rather to, given the weight matrices of a pretrained network, detect when it will fail to generalize.

\section{Mathematical introduction to an autoencoder\label{sec:model}}

In this Section, we give a brief overview of autoencoders. This exposition is not meant to be complete to a reader new to machine learning. The interested reader can fill in details from the provided references. A good reference for deep learning is \cite{Goodfellowbook}, where autoencoders are discussed in Chapter 14. It is worth noting that the exposition in this paper is aimed at a more mathematical audience than the exposition in that textbook, and glosses over many details of crucial importance to machine learning practitioners. 

For the purposes of this document, an autoencoder consists of three spaces, and input space, $\cI$, a reconstruction space $\cR$ and a latent space $\cL$ with two trained neural networks connecting them: an encoder mapping from the input space to the latent space, $E_{model} : \cI \rightarrow \cL$ and a decoder mapping from the latent space to the reconstruction space, $D_{model} : \cL \rightarrow \cR$. The input and reconstruction spaces are large dimensional real vector spaces $\cI \simeq \cR \simeq \R^N$ while the latent space is a smaller dimensional real vector space, $\cL = \R^d$ with $d << N$. For future reference, we call the full dataset $\cD$, while $\cD_{train}$ is the train dataset, and $\cD_{test} = \cD \setminus \cD_{train}$ is the test subset. We write the autoencoder $D_{model}\circ E_{model}$.

\begin{conj} \label{conj:datamanifold} The main conjecture underlying the construction of an autoencoder is that the data lies noisily on a lower dimensional subspace, $M_\cD$, embedded in $\cI$. \end{conj} 

The key point of this conjecture is the idea of \emph{dimensional reduction}. Namely, that the data, $\cD$, lives in a \emph{smaller} dimensional space than $\cI$. Thus one does not need all the information encoded in $\cI$ to understand the data. By mapping down to a smaller $d$ dimensional latent space, we hypothesize that the singular vectors of the derivative $\D E_{model}$ evaluated at each data point $x \in \cD$ locally extract $d$ most important features of the data manifold $M_\cD$. While we do not know the dimension of the subspace $M_\cD$, the conjecture is that the image of the encoder, $E_{model}$, is the projection of $M_\cD$ onto the latent space $\cL$ (if $\dim(\cL) <\dim(M_\cD)$) or the inclusion of $M_\cD$ into $\cL$ (if $\dim(\cL) >\dim(M_\cD)$). 

If Conjecture \ref{conj:datamanifold} holds, and then $E_{model}$ maps $\cD$ onto (a projection of) $M_\cD$ which contains all the relevant information to reconstruct the data. Then $D_{model}$ maps the dimensionally reduced data point back to the the original point. 

If the data were distributed exactly on $M_\cD$ (as opposed to noisily around $M_\cD$), and one knew the manifold $M_\cD$, the autoencoder can be written in terms of coordinate charts of $M_\cD$. Let $(\cU, \phi_\cU)$ be an atlas on $M_\cD$ and $\pi_\cU$ the projection (or embedding, if $\dim(\cL) > \dim(M_\cD)$ is large enough) of $\phi_\cU(\cU)$ onto $\cL$. Then the autoencoder can be represented by the following diagram: \ba \xymatrix{\cM_\cD \supset \cU_\cD \ar[r]^{\phi_\cU} \ar@{^{(}->}[d] & \R^{\dim(\cM_\cD)} \ar[dr]^{\pi_\cI} & & \R^{\dim(\cM_\cD)} \ar[dl]_{\pi_\cD} & \cU_\cD \ar[l]_{\phi_\cU}  \subset \cM_\cD \ar@{^{(}->}[d] \\ \cI  \ar[rr]^{E_{model}} &  & \cL \ar[rr]^{D_{model}}  & & \cR  }\;. \label{diag:AE}\ea Namely, the map $E_{model}$, when restricted to $M_\cD$ is the projection or embedding of the coordinate chart of $M_\cD$ onto $\cL$. 

\begin{dfn}\label{dfn:bigJac}
Let $J_\cI(x)$ be the Jacobian matrix of the map $D_{model} \circ E_{model}$ at the point $x \in \cI$. Let $\{\lambda_1(x), \ldots, \lambda_N(x)\}$ be the eigenvalues of $J_\cI(x)$. 
\end{dfn}

In this case, if $D_{model}\circ E_{model}$ preserves $M_\cD$, for each $x \in M_\cD$, the eigenvectors of nonzero eigenvalues of $J_\cI(x)$ will lie in the tangent space of $M_\cD$ in $\cR$: $T_{x}M_\cD$.  If $\dim(\cL) <\dim(M_\cD)$, we expect these eigenvectors to span a subspace of $T_{x}M_\cD$, while if $\dim(\cL) >\dim(M_\cD)$, we expect the eigenvectors will span $T_{x}M_\cD$.

Indeed, if $D_{model}\circ E_{model}(x)$ acts as the identity on $M_\cD$, then if $v$ is tangent to $M_\cD$ at $x$, $v$ will be preserved by $J_\cI(x)$; it will be an eigenvector of eigenvalue $1$.  If we also have that $\dim(\cL)=\dim(M_\cD)$, the other eigenvalues of $J_\cI$ would necessarily be $0$, so the Jacobian would be a projection onto $T_{x}M_\cD$.  If $\dim(\cL)>\dim(M_\cD)$ then there could be other nonzero eigenvalues even if $D_{model}\circ E_{model}$ perfectly preserves the data manifold. 

However, Conjecture \ref{conj:datamanifold} assumes that the data lies noisily around $M_\cD$. In this case, if one knows $M_\cD$, then, for all $x \in \cD$, the autoencoder maps the data point onto the data manifold. In other words \ba D_{model}\circ E_{model} (x) = x  + \varepsilon(x)\; \label{eq:AE}\ea where $D_{model}\circ E_{model} (x) \in M_\cD$. When we know $M_\cD$, the $L_2$ norm, $||\varepsilon(x)||_2^2$ is the distance of $x$ from the manifold.  For notational ease, sometimes, we denote $D_{model} \circ E_{model} (x) = x_{rec}$.

\begin{dfn} More generally, we say that $\varepsilon(x) =  x - D_{model}\circ E_{model} (x) \in \cR$ is the reconstruction error vector at the point $x$. The reconstruction error at this point is the $L_2$ norm square of this vector: $||\varepsilon(x)||^2_2$. \label{dfn:error}\end{dfn}

If $D_{model}\circ E_{model}$ only approximates the identity on $M_\cD$, then for $v$ a tangent vector to $M_\cD$ at the point $x$, there is no reason for the Jacobian map to preserve $v$.  In fact, if $D_{model}\circ E_{model}(x)=x+\varepsilon(x)$, the tangent spaces of $x$ and its image will not even be the same space, as they are ``based" at different points. However, we wish to identify these spaces in some way; we do this in the obvious manner by the usual identification of both spaces with $\R^N$.  This identification allows us to speak of eigenvectors of the Jacobian, $J_\cI$; these are simply the usual eigenvectors when we view the Jacobian as an $N\times N$ matrix.

From a geometric perspective, it would be more elegant to identify these tangent spaces in a manner compatible with the structure of $M_\cD$.  In particular, one might use the exponential map in the normal bundle of $M_\cD$, $NM_\cD$, to identify a small neighborhood of $M_\cD$ in $\R^N$ with a small neighborhood of $M_\cD$ in $NM_\cD$.  As long as $x$ and $x_{rec}$ are close enough to $M_\cD$ that they lie in this neighborhood, we could then identify both with points in $NM_\cD$, and identify the tangent spaces of $x$ and $x_{rec}$ by parallel transport along a geodesic between them, in $NM_\cD$.  In particular, if both $x$ and $x_{rec}$ happened to be in $M_\cD$, then this parallel transport would be along a geodesic of $M_\cD$.  We remark that the ``obvious" identification of tangent spaces in the previous paragraph comes from parallel transport between $x$ and $x_{rec}$ in $\R^N$.

This more geometric program poses a number of obstacles.  Firstly, if $x$ and $x_{rec}$ are both close to $M_\cD$ but not to each other, there may not be a unique length-minimizing geodesic between them in $NM_\cD$, and different choices of geodesic might produce different identifications of the tangent spaces.  Secondly, if $x$ and $x_{rec}$ were far from $M_\cD$, then the ``small neighborhood" of the previous paragraph might not contain $x$ and $x_{rec}$, so we would not be able to identify these with points in the normal bundle.  Fortunately, we conjecture that the point $x$ should lie quite close to the data manifold, and $x_{rec}$ should be quite close to $x$.  However, a final objection is harder to answer.  Because we do not know $M_\cD$, we cannot identify geodesics in $M_\cD$ or its normal bundle, and cannot parallel transport in these spaces either.  For this reason, we choose the simpler identification outlined above.


In short, the fact that we do not know $M_\cD$ means that we cannot determine local projections onto it. Therefore, we approximate the composition $\pi_\cU \circ \phi_\cU$ by a series of linear maps alternately composed with activation functions. The
coefficients of these linear maps are found by minimizing the reconstruction error \ba \sum_{x \in \cD_{train}} ||\varepsilon(x_0)||_2^2\;. \label{eq:mse}\ea given in Definition \ref{dfn:error}.

The trained encoder and decoder are two neural networks, $E_{model}$ and $D_{model}$ respectively, each defined as the composition of maps between $n+1$ vector spaces, $(\cI = \R^{N = l_0}, \R^{l_1}, \ldots , \cL = \R^{l_n = d})$ with \bas E_{model} = f_{e, n} \circ f_{e, n-1} \ldots \circ f_{e, 1} \eas and each $f_{e, i} : \R^{l_{i-1}} \rightarrow \R^{l_i}$.   Similarly, we can write $D_{model} = f_{d, n} \circ f_{d, n-1} \ldots \circ f_{d, 1}$ with the $f_{d, i}$ mapping between the vector spaces in the opposite order: $f_{d, i} : \R^{l_{n-i}} \rightarrow \R^{l_{n-i+1}}$. Each of the $f_{enc, i}$ and $f_{dec, i}$, for $i < n-1$ are defined as a composition of an affine transformation with an activation function, while the last function is simply an affine transformation: \bas f_{e, i}(x) = \begin{cases}R_i(A_ix + b_i) & i < n-1 \\ A_ix + b_i & i = n \;,\end{cases} \eas and similarly for $f_{d,i}$. In other words, $A_i$ is a $l_i \times l_{i-1}$ dimensional matrix and $b_i$ as vector in $\R^{l_i}$. The activation function $R_i$ is a component wise implementation of some activation function (such as ReLU, softmax, sigmoid, etc.). 

\begin{dfn} \label{dfn:autoencoder}
A trained autoencoder is defined by the tuple $(\cI, E_{model}, \cL, N_{dec}, \cR)$ as defined above. 
\end{dfn} 

Note that it is only after the the auto encoder is trained via back propagation that the functions $E_{model}$ and $D_{model}$ are fixed. For the purposes of this paper, we are not interested in improving training techniques, only in evaluating the performance of a fully trained pair of networks. In the sequel, we drop the adjective \emph{trained} when referring to this system.

Finally,  we will mostly be interested in local properties of the map $D_{model}\circ E_{model}(x)$. For instance, we will study the local Jacobian, denoted $J_\cI(x)$ and a vector of eigenvalues of said matrix $\vec \lambda_{\cI}(x)$. However, for simplicity of of notation, we will omit the input point $x$. We attempt to retain in cases where it improves the clarity.

\section{Local distance from the identity and MSE\label{sec:math}}

In this paper, we wish to understand how far the trained autoencoder, represented by the map $D_{model} \circ E_{model}$, differs from a projection map near each point in $\cD$. In training an autoencoder, one minimizes the reconstruction loss on $\cD_{train}$ and hopes that the distribution of $\cD$ is such that these maps are inverses to each other {\em on more than just on the training points}. Indeed, we see that the mean squared error on the test points continues to be very low.

In order to understand the local behavior, one must look at more than just the evaluation on the data point. In particular, one must go beyond just the mean squared error and consider the Jacobian of the function $D_{model} \circ E_{model}$ on $\cD$. In particular, we conjecture a relationship between local distance from the identity matrix on the training point and the MSE on the near by test points.

\begin{conj}\label{conj:notid}
The further the map $D_{model} \circ E_{model}$ is from a projection near the training points, the worse the MSE will be for the autoencoder on the test points.
\end{conj}

At times in this paper, we study the effect of the two parts of the autoencoder on the latent space. That is, we study the effect of $E_{model} \circ D_{model}$ locally at the point $ y = E_{model}(x) \in \cL$. This is \emph{not} the map defined by the autoencoder. Rather, it is the map induced by trained autoencoder on the latent space $\cL$ by reversing the order of composition. However, as we show in Section \ref{sec:forms}, the geometry of this map is much easier to interpret. Therefore, we introduce a new Jacobian matrix.

\begin{dfn}
For $ y = E_{model}(x)$, define $J_\cL(y)$ to be the Jacobian matrix of the map $E_{model} \circ D_{model}(y)$. 
\end{dfn}

A natural way to consider how far a function is from a projection, locally, is to consider its Jacobian matrix.

We may now ask how far $J_\cI(x)$ is from a projection, and how far $J_\cL(x)$ is from the identity matrix $\id_d$. First we define the Frobenius norm of a matrix.

\begin{dfn}
For a matrix $A$, let \[
\|A\|_F^2:=\sum_{i=1}^m\sum_{j=1}^n|a_{ij}|^2
\] indicate the Frobenius (or Hilbert-Schmidt) norm of the matrix $A$.  

This gives us a notion of a distance between matrices, the Frobenius-norm distance, which is defined as $d_F(A,B):=\|A-B\|_F$
\label{dfn:FrobNorm}\end{dfn}

In Section \ref{sec:Frobbound} we show that the Frobenius distance of the matrix $J_\cI$ to a certain associated projection bounds the MSE of the autoencoder, as long as $J_\cI$ has eigenvectors which are close enough to orthogonal. 

We also consider the eigenvalues and eigenvectors of $J_\cL(y)$ and $J_\cI(x)$. In particular, in Section \ref{sec:forms} we discuss the eigenvalues of the Jacobians $J_\cI(x)$ and  $J_\cL(y)$ as another measure of the how far these matrices are from being projections onto the appropriate tangent spaces. 

\subsection{Frobenius Norm and MSE bounds\label{sec:Frobbound}}

In this Section, we show that the distance from the matrix $J_\cI(x)$ to a certain associated projection, as defined by the Frobenius norm, gives an upper bound for the MSE for the autoencoder, as long as the nontrivial eigenvectors are sufficiently close to being orthogonal. We begin with a few observations about the reconstruction error.

First we note that, given a training point $x \in \cD_{train}$ and a test point $y \in \cD_{test}$, we can write the difference in reconstruction errors in terms of the first order Taylor series of $D_{model} \circ E_{model}$ around $x$. 

\begin{lem}\label{res:firstorder}
Given a training point $x \in \cD_{train}$ and a test point $y \in \cD_{test}$, 
\[ \varepsilon(y)  - \varepsilon(x) \approx  x -y + J_\cI(x) (y-x) \;.\]
\end{lem}
\begin{proof}
Since $y$ is close to $x$, to first order, 
\[ D_{model} \circ E_{model}(y) \approx D_{model} \circ E_{model}(x) + J_\cI(x) (y-x) \;.\] For notational convenience, call $x_{rec} = D_{model} \circ E_{model}(x)$  and $y_{rec} = D_{model} \circ E_{model}(y)$ the reconstruction of the points $x$ and $y$. Subtracting $x$ and $y$ from both sides, we see \[ y_{rec}  - x - y\approx x_{rec} - x -y + J_\cI(x) (y-x) \;\] which we can rewrite \[ \varepsilon(y)  - \varepsilon(x) \approx  x -y + J_\cI(x) (y-x) \;.\]
\end{proof}

We may use this to decompose the MSE of the test point $y$ into three components, each coming either from a different part of the autoencoder or Conjecture \ref{conj:datamanifold}. 

\begin{lem}\label{res:decompose}
Given a training point $x \in \cD_{train}$ and a test point $y \in \cD_{test}$ one may decompose $\varepsilon(y)$ into three parts,  
one coming from the distortion of caused by the autoencoder, one that comes from the projection map $\pi_\cI$ in display {\eqref{diag:AE}}, and one coming from the  MSE of the nearby training point.
\end{lem}

\begin{proof}

First, we decompose $x-y$ into the eigenvectors of $J_\cI(y -x)$,
\[ y-x=\sum_{i=1}^Nc_i v_i  \label{second} \; , \]
so that 
\[ J_\cI(x)(y-x)=\sum_{i=1}^N\lambda_ic_i v_i\; . \label{third} \] 
Here, the $\lambda_i$ are the eigenvalues of $J_\cI(y -x)$. Since $\dim(\cL) = d$,  $\lambda_i=0$ for all $i>d$.

Putting these together with Lemma \ref{res:firstorder}, we get 

\[ \varepsilon(y)  - \varepsilon(x) \approx  \sum_{i=1}^N\lambda_ic_iv_i-\sum_{i=1}^Nc_iv_i =\sum_{i=1}^N(\lambda_i-1)c_iv_i =  \sum_{i=1}^d(\lambda_i-1)c_iv_i - \sum_{i=d+1}^Nc_iv_i\;.\]
 
Since the $\lambda_i$ are zero for $i>d$, we rewrite this approximation as
\[
\varepsilon(y) \approx\underbrace{\sum_{i=1}^d(\lambda_i-1)c_iv_i}_{(A)}-\underbrace{\sum_{i=d+1}^N c_iv_i}_{(B)} +  \varepsilon(x) .
\]

Thus we have three separate sources of error: the error arising from the distortion on the image of the autoencoder, $(A)$; the distance from how far $x$ is from the image of the autoencoder, $(B)$; and the original MSE on the training point $x$, $\varepsilon(x)$. 

\end{proof}

Note that as the latent dimension increases, in particular as it surpasses the dimension of the data manifold, $M_\cD$, the term (B) should also become small. In this case, the difference in the reconstruction error vectors is approximated by \bas \varepsilon(y) - \varepsilon(x)\approx  \sum_{i=1}^d(\lambda_i-1)c_iv_i \;. \eas 
 
We are now ready to give a bound for the difference in the test and training data point reconstruction errors in terms of the Frobenius norm of $J_\cI$. We begin with a few definitions. 

\begin{dfn}For $A$ an $m \times n$ matrix, the \emph{singular values} are the square roots of the nonzero eigenvalues of the matrix $AA^T$. \label{dfn:singular} \end{dfn} 

Recall that the Frobenius norm of a matrix can equivalently be defined by the equation
\[
\|A\|_F^2=\sum_{i=1}^{\min{m,n}}\sigma_i(A)^2
\]
where $\sigma_i(A)$ denotes the $i^{\text{th}}$ singular value of $A$.  

For an $n\times n$ matrix $A$, the following inequality holds for any $p>0$ and any $1\leq k\leq n$, where $|\lambda_1|\geq|\lambda_2|\geq\dots\geq|\lambda_n|$ and $\sigma_1\geq\sigma_2\geq\dots\geq\sigma_n$:
\begin{equation}\label{eq:Weyl}
\sum_{i=1}^k|\lambda_i|^p\leq\sum_{i=1}^k\sigma_i^p.
\end{equation}
This is a consequence of Weyl's Majorization Theorem.

\begin{dfn} Let $A$ be a square matrix which is diagonalizable over $\C$. Let $v_1,\dots,v_k$ be the eigenvectors corresponding to nonzero eigenvalues and $v_{k+1},\dots,v_N$ those corresponding to the eigenvalue $0$ (that is, $Av_{i}=0$ for $i=k+1,\dots,N$).  We define an oblique projection, $P_A$, by
\[
\left\{
\begin{array}{rl}
P_Av_i=v_i & 1\leq i\leq k\\
P_Av_i=0 & k+1\leq i\leq N.
\end{array}
\right.
\] \label{dfn:oblique}\end{dfn}

Like an orthogonal projection, the eigenvalues of $P_A$ are either $0$ or $1$. However, unlike an orthogonal projection, the oblique projection shares eigenvectors with $A$.  

We also remark that because almost all matrices are diagonalizable over $\C$, the Jacobian of an autoencoder is diagonalizable over $\C$ with probability $1$.  In the following theorem and the remainder of this paper, we thus consider the input and reconstruction spaces to be isometrically embedded in $\C^N$, which allows us to diagonalize $J_\cI$ (without changing the MSE).

We also recall that for two vectors in $\C^N$, $v=\sum_{i=1}^Nv_ie_i$ and $w=\sum_{i=1}^Nw_ie_i$, where $e_i$ is the standard basis of $\C^N$, their inner product is defined by
\[
\langle v,w\rangle=\sum_{i=1}^Nv_i\overline{w_i}.
\]
In particular, we observe that $\langle av,bw\rangle=a\overline{b}\langle v,w\rangle$ for any complex numbers $a,b$.  As usual, we have that $\|v\|^2=\langle v,v\rangle$.

Now we are ready to state the main theorem of this Section.

\begin{thm}\label{res:bound}
Let $J_\cI$ be the Jacobian of an autoencoder with a $d$ dimensional latent space. Let $v_1,\dots,v_d$ be a set of unit eigenvectors for the $d$ largest eigenvalues.  Suppose that $|\langle v_i, v_j\rangle|\leq\delta \leq \frac{1}{2(d-1)}$ for all $i \neq j$. Then if Conjecture \ref{conj:datamanifold} holds, for large enough $d$, the quantity $3\|J_\cI-P_{J_{\cI}}\|_F^2\|y-x\|^2$ bounds the magnitude of the difference in reconstruction error vectors where $\|y-x\|^2$ denotes the square of the distance between the test and training point.
\end{thm}

\begin{proof}

From Lemma \ref{res:decompose}, the MSE on the test point is given by 
\[ 
||\varepsilon(y)-\varepsilon(x)||_2^2 \approx ||\sum_{i=1}^d(\lambda_i-1)c_iv_i - \sum_{i=d+1}^Dc_iv_i||_2^2  \;.\]

By Conjecture \ref{conj:datamanifold}, we expect our dataset to lie close to a manifold of small dimension. Therefore, we expect the error term $(B)$ from Lemma \ref{res:firstorder} to be small if the latent dimension is larger than the dimension of this manifold.  Hence, we expect that \ba ||\varepsilon(y)-\varepsilon(x)||_2^2 \approx ||\sum_{i=1}^d(\lambda_i-1)c_iv_i - \sum_{i=d+1}^Dc_iv_i||_2^2   \\ \approx  \sum_{i,j=1}^d \langle v_i, v_j\rangle c_i\overline{c_j}(\lambda_i-1)\overline{(\lambda_j-1)}  \;. \label{eq:mainapprox}\ea 

Because we are working over $\C$, some of our eigenvalues and eigenvectors may be complex, so some of the constants $c_i$ may also be complex.  Hence, we conjugate these when we factor them out of the second term of the inner product.

We now consider the off-diagonal terms of this sum, bounding these by the assumption that $|\langle v_i, v_j\rangle|\leq\frac{1}{2(d-1)}$ and the fact that $|ab|\leq \frac{|a|^2+|b|^2}{2}$, to conclude that
\[
\|\sum_{i\neq j}\langle v_i, v_j\rangle c_i\overline{c_j}(\lambda_i-1)\overline{(\lambda_j-1)}\|\leq\frac{1}{2(d-1)}\sum_{i\neq j}\frac{c_i^2|\lambda_i-1|^2+c_j^2|\lambda_j-1|^2}{2}.
\]
Now, each fixed $i$ appears in $2(d-1)$ pairs of the form $(i,j)$ where $i \neq j$.  Hence, when we group like terms together, we compute that
\[
\|\sum_{i\neq j}\langle v_i, v_j\rangle c_i\overline{c_j}(\lambda_i-1)\overline{(\lambda_j-1)}\|\leq\frac{1}{2}\sum_{i=1}^dc_i^2|\lambda_i-1|^2.
\]
Combining the off-diagonal terms with the diagonal terms, we have
\[
\|\sum_{i, j=1}^d\langle v_i, v_j\rangle c_i\overline{c_j}(\lambda_i-1)\overline{(\lambda_j-1)}\|\leq\frac{3}{2}\sum_{i=1}^dc_i^2|\lambda_i-1|^2.
\]

In particular, if we let $C=\max\{|c_i|^2:1\leq i\leq d\}$, we see that
\[
\sum_{i=1}^d |\lambda_i-1|^2c_i^2\leq C\sum_{i=1}^d|\lambda_i-1|^2.
\]

We bound $C$ in terms of $\|y-x\|^2$.  Recalling that $y-x=\sum_{i=1}^dc_iv_i+\sum_{i=d+1}^Nc_iv_i$, and that this latter term (analogous to term (B) in our computations above) should become negligible, we find that
\[
\|y-x\|^2\approx\|\sum_{i=1}^dc_iv_i\|^2=\sum_{i=1}^dc_i^2\|v_i\|^2+\sum_{i\neq j}c_i\overline{c_j}\langle v_i,v_j\rangle
\]
We can repeat the same arguments as above to bound the absolute value of the second, off-diagonal term by $\frac{1}{2}\sum_{i=1}^dc_i^2$, which means that
\[
\|y-x\|^2\geq\frac{1}{2}\sum_{i=1}^dc_i^2.
\]

In particular, we note that $C=\max\{|c_i|^2:1\leq i\leq d\}$ is certainly less than or equal to $\sum_{i=1}^dc_i^2$, and hence is bounded by $2\|y-x\|^2$.

Finally, we note that by the definition of $P_{J_\cI}$, the $\lambda_i-1$ are the eigenvalues of $J_\cI-P_{J_\cI}$.  Indeed, if $v_i$ is an eigenvector of $J_\cI$ corresponding to the eigenvalue $\lambda_i\neq0$, then $v_i$ is an eigenvector of $J_\cI-P_{J_\cI}$ corresponding to the eigenvalue $\lambda_i-1$.  Moreover, if $v_j$ is an eigenvector of $J_\cI$ for the eigenvalue $0$, then by the definition of $P_{J_\cI}$ it is an eigenvector of $J_\cI-P_{J_\cI}$ corresponding to the eigenvalue $0$. In particular, recalling equation (\ref{eq:Weyl}), with $p=2$ and $k=d$, we have that
\[
\sum_{i=1}^d|\lambda_i-1|^2\leq\sum_{i=1}^d\sigma_i(J_\cI-P_{J_\cI})^2\leq\|J_\cI-P_{J_\cI}\|_F^2.
\]
Here, we have used the fact that the non-zero eigenvalues are contained in ther first $d$ eigenvalues of $J_\cI-P_{J_\cI}$ ordered by magnitude, as the remainder are zero.

In particular, if the latent dimension is large enough that error term $(B)$ from Lemma \ref{res:firstorder} is small and term $(A)$ is the main contributor to error, we have by equation \eqref{eq:mainapprox}
\[
||\varepsilon(y)-\varepsilon(x)||_2^2\leq C^2(\delta(d-1)+1)\|J_\cI-P_{J_\cI}\|_F^2 \;.
\]
\end{proof}

In general, we note that the portion of the error so bounded will increase in the latent dimension.  The constant $C$ depends on the point $y$, so this bound becomes less useful the further we are from the training point $x$.

The Frobenius norm is a commonly used tool for regularizing neural networks, where it is commonly known as weight decay regularization, where it is used to guard against overfitting. See \cite{weightdecay} for a good discussion on the topic. However, unlike most neural networks, autoencoders have the advantage of having a square Jacobian matrix, $J_\cI(x)$, which we can compare to a projection map. We leave the refinement of Theorem \ref{res:bound} to future work, as it needs an ability to estimate the orthogonality of the eigenvectors $v_i$.  Furthermore, understanding the structure of the eigenvectors $v_i$ may give insight into the quality of the feature extraction aspect of an autoencoder. 

\subsection{Sums and products of eigenvalues \label{sec:forms}}

We see from Lemma \ref{res:firstorder} that the first order constraints on the autoencoder help us control the reconstruction error. It assures that the function $D_{model}\circ E_{model}$ is as close to the identity as possible on the test data points. The Jacobian matrix $J_\cI(x)$ is necessary to calculate the first order approximation of this autoencoder to the identity, when restricted to a neighborhood of $M_\cD$. 

More broadly, we are interested in how closely the encoder and decoder functions invert each other.  Even if $D_{model}\circ E_{model}(x)=x$, which says that these invert each other at the zeroth order at $x$, we can ask how close they are at higher orders.  To first order, this asks how close $J_\cI(x)$ is to being the identity.  Because this matrix factors through the latent space, it cannot be the identity on all of $T_x\cI$, as $\dim(\cL)<<\dim(\cI)$.  Hence, the most we could possibly ask is for $J_\cI(x)$ to act as the identity on its image.  In this case, the eigenvalues of $J_\cI(x)$ would be either $1$ or $0$, with $\lambda_1=\lambda_2=\dots=\lambda_d=1$, and the rest equal to zero.

In particular, we inspect the $d$ largest eigenvalues of $J_\cI(x)$ (noting that by construction, the $N -d$ smallest eigenvalues will be $0$), $\lambda_1,\lambda_2,\dots,\lambda_d$.  We propose a number metrics to compare these to $1$ in an averaged sense.  The product of these eigenvalues, $\prod_{i=1}^d\lambda_i$, measures how much $J_{\cI}(x)$ distorts the $d$ dimensional volume on its image, and thus has a somewhat geometric interpretation.  The sum of the eigenvalues, which is the trace $\Tr(J_\cI(x))$, is computationally easier to evaluate, although its geometric meaning is less easily described.  One could also consider other symmetric polynomials of the eigenvalues, which arise as the coefficients of the characteristic polynomial of $J_\cI(x)$.  However, these are less easily computed than the trace and lack the natural geometric interpretation that the product has. Therefore, for the rest of the document, we only focus on the sum and the product of the eigenvalues. 

One may look also at the eigenvalues of $J_\cL(y)$, with $y = E_{model}(x)$. The product of these $d$ eigenvalues gives the determinant of $J_\cL(y)$, which is the volume form on $T_y\cL$. This quantity has the computational advantage that the determinant is faster to calculate than the set of $d$ eigenvalues. It also has the conceptual advantage of giving a measurement of the magnitude of local distortion of the $\cL$ at the point $y = E_{model}(x)$. If $\det(J_\cL(y)) \neq 1$, then $E_{model}\circ D_{model}(x)$ distorts the tangent space $T_y\cL$. The further the quantity $\log(\det(J_\cL(y))$ is from $0$, the greater the distortion. In other words, instead of measuring the distortion of $T_xM_\cD$ for a hypothesized data manifold $M_\cD$, this measure concretely studies the distortion on the latent space. 

\begin{dfn}
Let $\vec{\lambda}_\cI$ and $\vec{\lambda}_\cL$ be the $d$ nonzero eigenvalues of $J_\cI$ and $J_\cL$ respectively, arranged in decreasing order by absolute value.  That is, if $\vec{\lambda}_\cI=\{\lambda_1,\lambda_2,\dots,\lambda_d\}$, then we have that $|\lambda_1|\geq|\lambda_2|\geq\dots\geq|\lambda_d|$. Similarly, we define $\omega_\cL(z) = \prod_{\lambda_i \in \vec{\lambda}_\cL(z)} \lambda_i$ and 
$\omega_\cI(x) = \prod_{\lambda_i \in \vec{\lambda}_\cI(x)} \lambda_i$ to be the products of these nonzero eigenvalues. 
\end{dfn}

 Note that $\omega_\cL(y)$ is the top form defined on $T_y\cL$ induced by the map $E_{model}\circ D_{model}(y)$, while $\omega_\cI(x)$ quantifies the extent to which the  map $D_{model}\circ E_{model}$ distorts the volume of the $d$ dimensional span of its Jacobian's nontrivial eigenvectors. 

Furthermore, we note that these two sets of eigenvalues are related. In particular, if $D_{model}\circ E_{model}(x)=x$ then, by the chain rule, the eigenvalues of $J_\cL(z)$ are the same as those of $J_\cI(x)$, where $y=E_{model}(x)$.

\begin{prop} \label{res:evalssame}
If $D_{model}\circ E_{model}(x)=x$, then for $y = E_{model}(x)$, the $d$ nonzero eigenvalues of $J_\cI(x)$ are the same as the eigenvalues of $J_\cL(y)$. 
\end{prop} 
\begin{proof}
This is a result of the chain rule.

Since $D_{model}\circ E_{model}(x) = x$, let $w$ be an eigenvector of $J_\cI(x)$ with eigenvalue $\lambda$. Then \bas \D D_{model}|_z\circ \D E_{model}|_x(w) = \lambda w \;, \eas where $y = E_{model}(x)$ as above. If $v = \D E_{model}|_x(w)$ then \bas  \D E_{model}|_x \circ \D D_{model}|_z (v) =  \D E_{model}|_x \circ \D D_{model} \circ \D E_{model}|_x (w) =  \D E_{model}|_x  (\lambda w) = \lambda(v)\;. \eas
\end{proof}

We remark that this result continues to hold if $J_\cI$ has rank less than $d$, and the proof method is similar (requiring merely that we handle the zero eigenvectors separately).  

Unfortunately, the hypothesis of Proposition \ref{res:evalssame} does not generally hold, because an autoencoder may not achieve zero error on the training points.  
However, empirically, we see that these eigenvalues are not far from each other. We observe in Figure \ref{fig:eigennorms}\footnote{Here and in all figures, we only present data from a particular training epoch and seed. The data from other epochs and seed are similar and omitted for visual simplicity.} that quantities $\frac{||\vec{\lambda}_\cI- \vec{\lambda}_\cL||_1}{\dim(\cL)}$, and $\frac{||\vec{\lambda}_\cI- \vec{\lambda}_\cL||_2}{\dim(\cL)}$ i.e. the $L_1$ and $L_2$ norms of the differences in between the ordered eigenvalues over the dimension of the latent space is small, and that the standard deviations decrease with latent dimension.

\begin{figure} 
\begin{subfigure}[t]{0.3\textwidth}
\includegraphics[width=\textwidth]{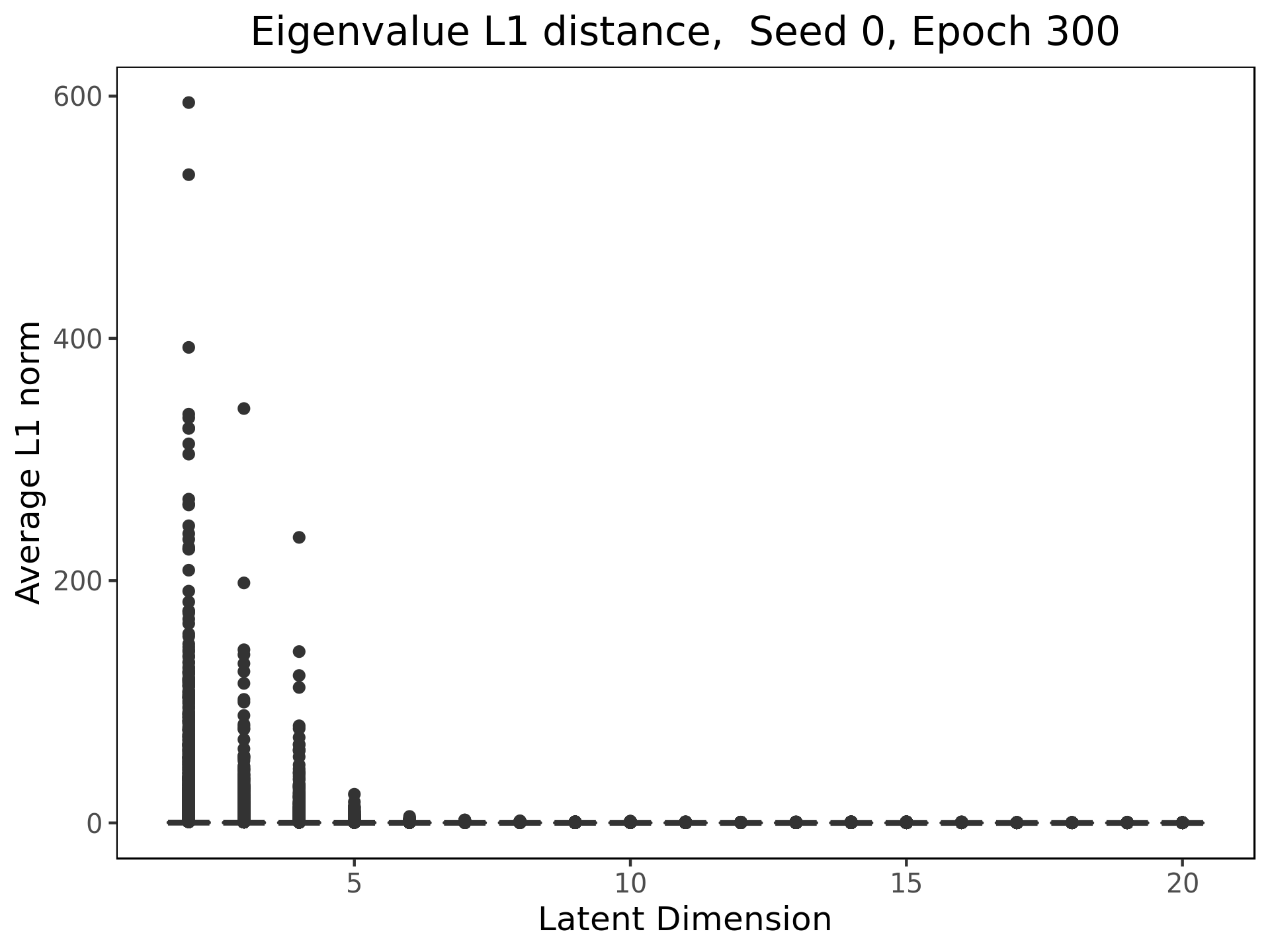}
\end{subfigure}
\begin{subfigure}[t]{0.3\textwidth}
\includegraphics[width=\textwidth]{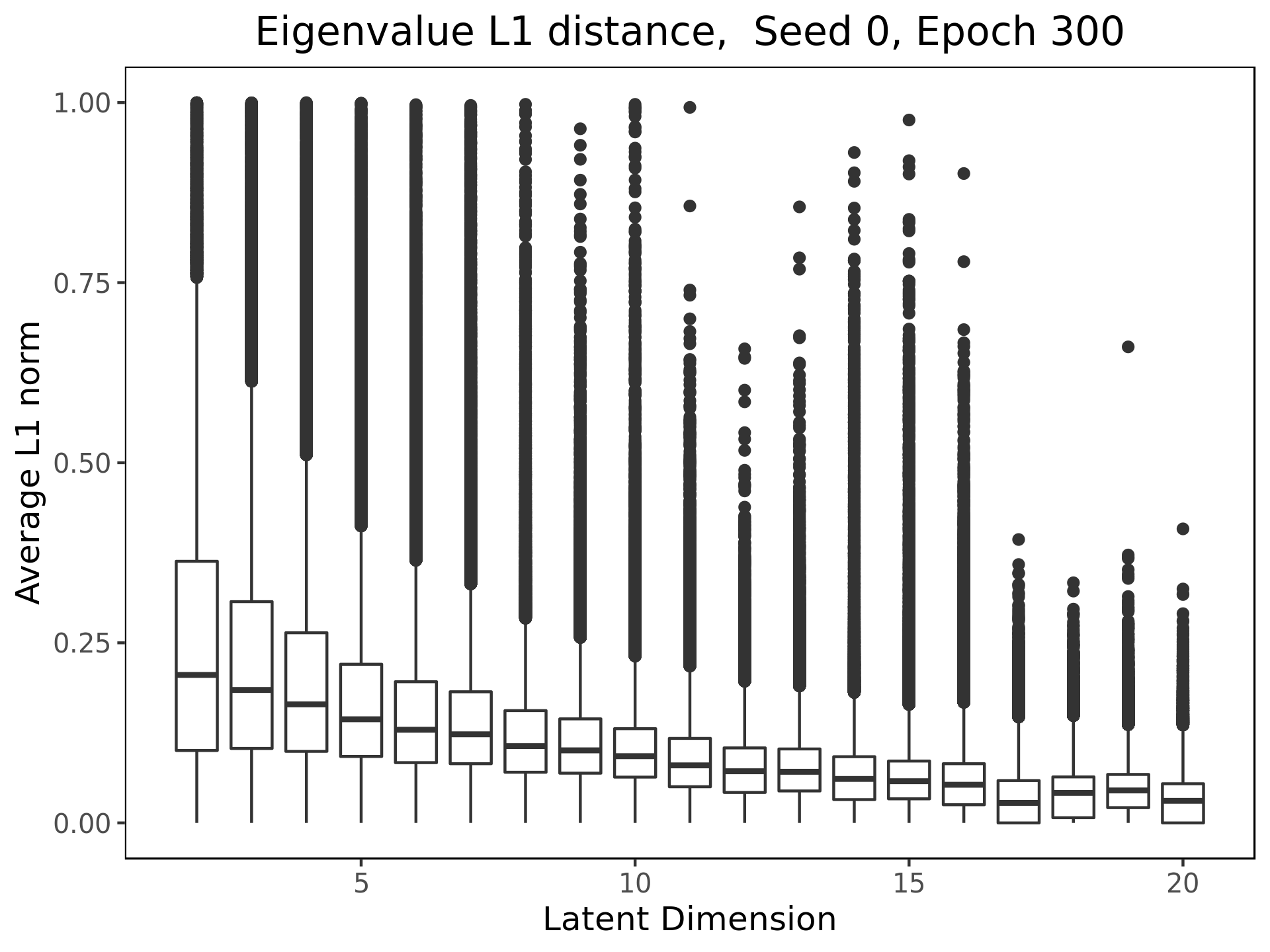}
\end{subfigure}
\newline 
\begin{subfigure}[t]{0.3\textwidth}
\includegraphics[width=\textwidth]{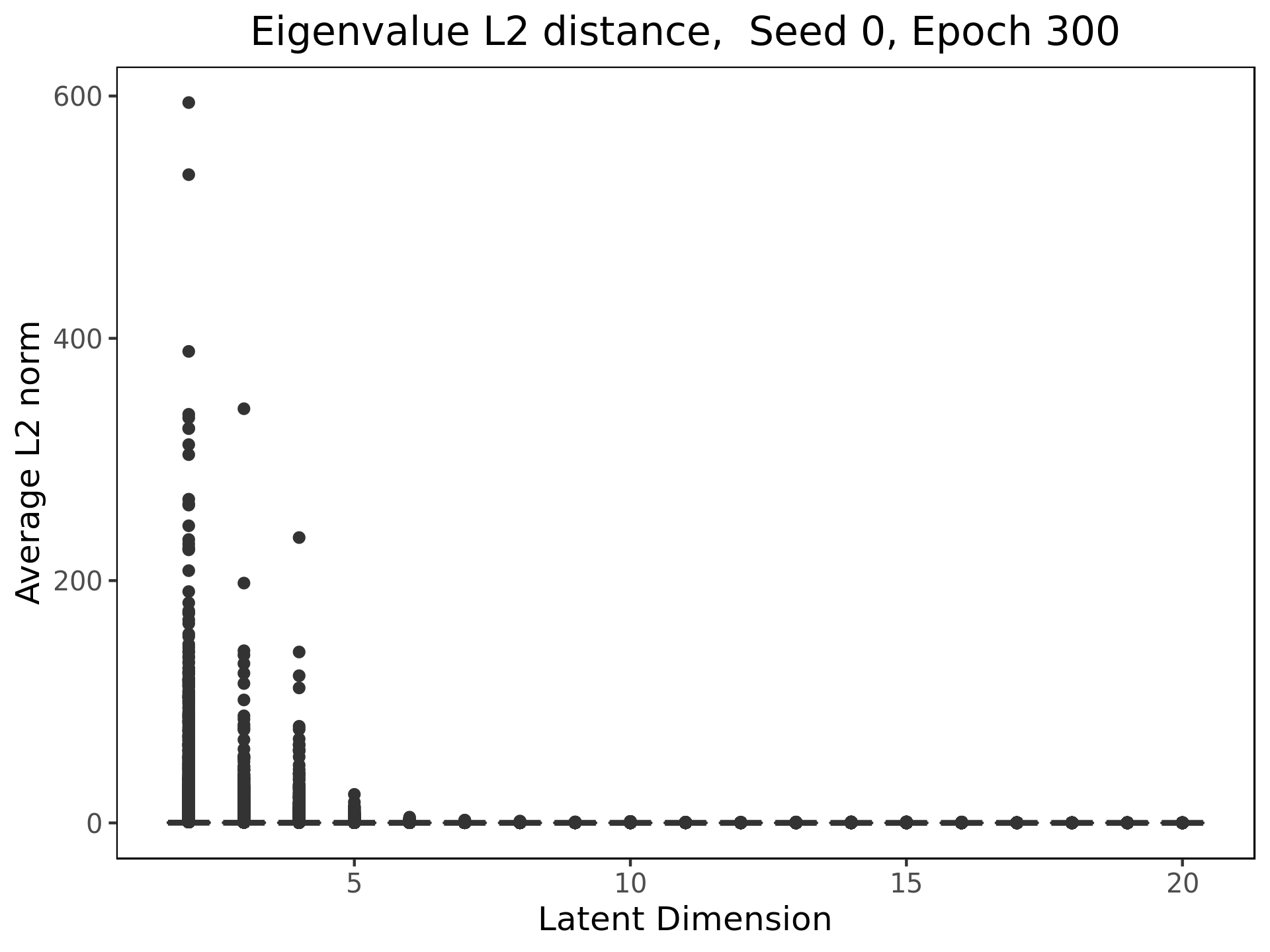}
\end{subfigure}
\begin{subfigure}[t]{0.3\textwidth}
\includegraphics[width=\textwidth]{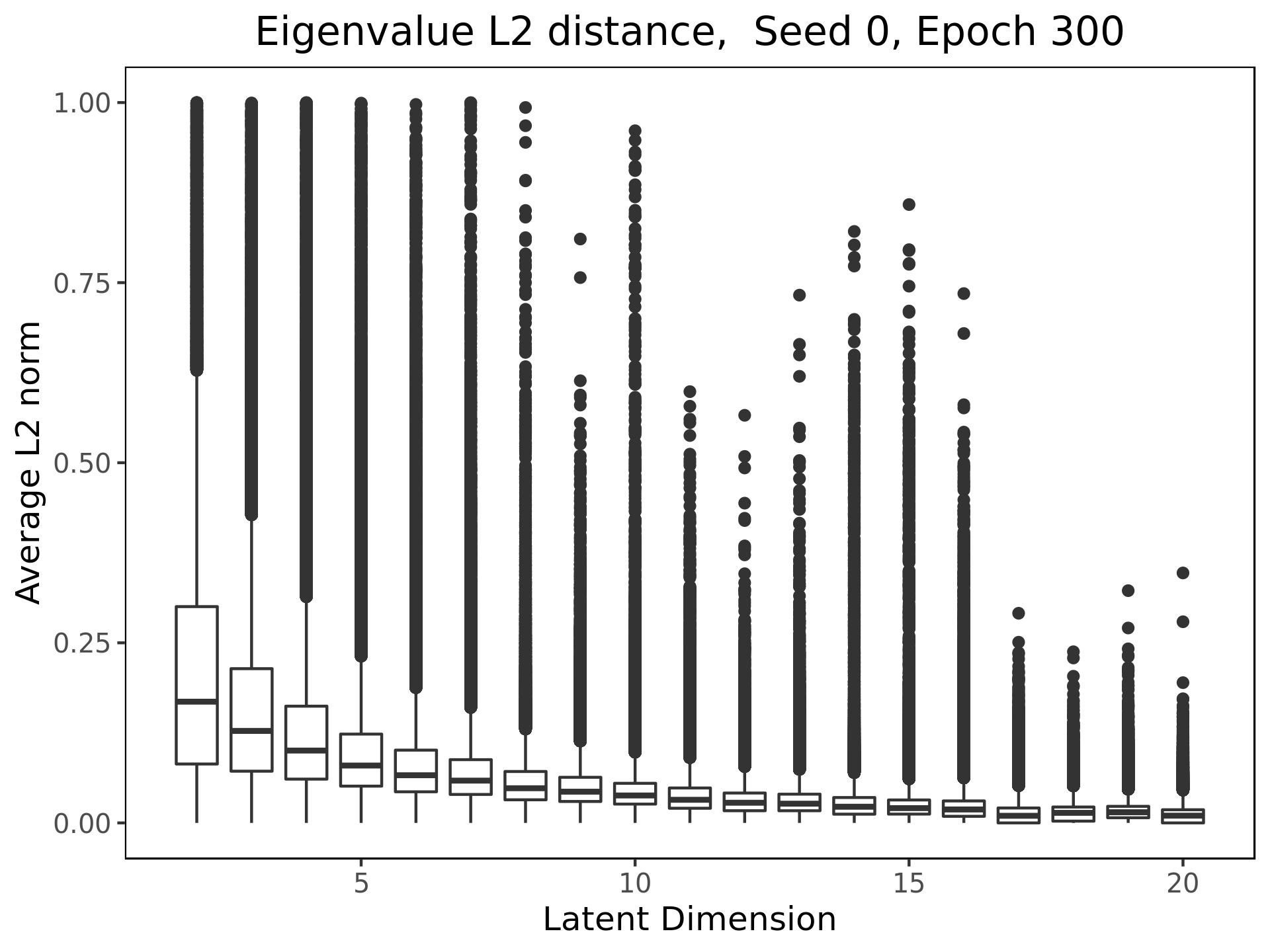}
\end{subfigure}

\caption{The ratio of the $L_1$ norm of the difference in eigenvalues to the latent dimension (top) and the ratio of the $L_2$ norm of the difference in eigenvalues to the latent dimension (bottom) is small and decreases both in median and standard deviation as latent dimension increases. } \label{fig:eigennorms}
\end{figure}

\section{Experimental Results \label{sec:results}}

In this Section, we present the empirical findings on the eigenvalues of the autoencoder. We also show that not only are $\Tr(J_\cL)$, $\Tr(J_\cI)$, $\omega_\cI$ and $\omega_\cL$ good predictors of the reconstruction error, in expectation, as these quantities increase, the expected reconstruction error for test points is higher than the expected reconstruction error for training points. In other words, these quantities are good predictors of where a network will fail to generalize. 

\subsection{Architectures \label{sec:setup}}
We trained a large set of autoencoder structures across many different seeds on the MNIST dataset to test our conjectures. This Section describes the specifics of the various architectures.

The dimension of the input and reconstruction space in this case is $\dim(\cI) = \dim(\cR) = 784$, since MNIST images are 28x28 and the images are flattened to a vector before running the autoencoder. We developed a set of 3 layer, 4 layer and 5 layer autoencoders, each with latent dimension varying from $2$ to $20$, and trained each of the resulting set of 57 different autoencoders across 10 different random seeds. All experiments performed for this paper are on fully connected. For all autoencoders, there are ReLU activation functions between each layer of the encoder and decoder, and there is a $\tanh$ activation function after the final layer of decoder to normalize the outputs to match the inputs. All autoencoders are trained using random batches of size 128 for 500 epochs, and parameters are saved every 100 epochs to measure changes in behavior through training. The Adam optimizer was used with default parameters for learning rate and epsilon, and no weight decay was used.  The inputs were normalized to be in the range $[-1, 1]$.

The first base autoencoder structure, exp-3-layer, has three layers in the encoder and decoder. The encoder half consists of layers $(\cI = \R^{784}, \R^{128}, \R^{32}, \R^{d})$, and the decoder half consists of layers $(\R^{d}, \R^{32}, \R^{128}, \cR = \R^{784} )$. The second autoencoder structure, exp-default, has four layers in the encoder and decoder. The encoder half consists of layers $(\cI = \R^{784}, \R^{128}, \R^{64}, \R^{32}, \R^{d})$, and the decoder half consists of layers $(\R^{d}, \R^{32}, \R^{64}, \R^{128}, \cR = \R^{784} )$. The final autoencoder structure, exp-5-layer, has 6 layers in the encoder and decoder. The encoder half consists of layers $(\cI = \R^{784}, \R^{128}, \R^{96}, \R^{64}, \R^{32}, \R^{d})$, and the decoder half consists of layers $(\R^{d}, \R^{32}, \R^{64}, \R^{96}, \R^{128}, \cR = \R^{784} )$.

\subsubsection{Explicit computations of Jacobians \label{sec:explicitJacobian}}
The computation of the Jacobians $J_\cI$ and $J_\cL$ is done by applying the chain rule sequentially forward through both the encoder and decoder layers. Note that the derivative of a layer is applied to the output from the last layer:
\begin{equation*} \label{eq:jacobian}
    F = F_n \circ \cdots \circ F_1, \textrm{ then } \D F  = \prod \D F_i |_{F_{i-1}\circ \cdots \circ F_1(x)} \;.
\end{equation*}
Explicitly the derivative of a affine transformation is the associated matrix is 
\begin{equation*} \label{eq:d_affine}
    \frac{\partial (Ax + b)}{\partial x} = A \;.
\end{equation*}
Recall that the ReLU function is, componentwise the piecewise function \bas ReLU(x) = \begin{cases} 0 & x <0 \\ x & x \geq 0 \end{cases}\;.\eas The derivative of the ReLU function applied to an $N$ dimensional vector is the Heaviside function embedded along the diagonals of an $N \times N$:
\begin{equation*} \label{eq:d_relu}
    \frac{\partial ReLU(x)}{\partial x} = H(x) =
    \begin{cases}
        1, & x \geq 0 \\
        0, & x < 0
    \end{cases} \;.
\end{equation*} and the derivative of the $\tanh$ function applied to an $N$ dimensional vector is the derivative embedded along the diagonals of an $N \times N$ dimensional matrix is 
\begin{equation*} \label{eq:d_tanh}
    \frac{\partial \tanh(x)}{\partial x} = 1 - \tanh^2(x)
\end{equation*}
Note that in the autoencoder, the ReLU and the $\tanh$ are applied component wise to the outputs of the previous layer. Thus the corresponding term in the Jacobian calculation is a matrix with the functions $H(x)$ or $\tanh(x)$ along the diagonals.

Explicitly, for the first autoencoder structure, exp-2-layer, we may write
\bas f_{enc, i}(x) = 
\begin{cases}
ReLU_{128}(A_{128 \times 784} x + b_{128}) & i = 1 \\ 
ReLU_{32}(A_{32 \times 128} x + b_{32} ) & i = 2 \\
A_{d \times 32} x + b_d  & i = 3 \;,
\end{cases}
\quad
;
\quad f_{dec, i}(x)
\begin{cases}ReLU_{32}(A_{32 \times d} x + b_{32}) & i = 1 \\
ReLU_{128}(A_{128 \times 32} x + b_{128} ) & i = 2 \\
\tanh_{784}(A_{784 \times 128} x + b_{784})  & i = 3 \;,
\end{cases}
\eas
where $A_{n \times m}$ is a real $n \times m$ matrix, $b_m$ is a $m$ dimensional column vector, and $ReLU_m : \R^m \rightarrow \R^m$ is the vector valued function that is $ReLU$ in each component (similarly for $\tanh_m$). 




As a particular example, we can compute that the derivative of $f_{enc,1}$ is
\[
\D f_{enc,1}(x)=H(A_{128\times784}x+b)\bI_{128}A_{128\times784}
\]
and the other derivatives can be computed similarly, as can the derivatives of compositions (using the chain rule).

\subsection{Eigenvalues \label{sec:evals}}
For the rest of this paper, we refer to the Jacobians $J_\cL$ and $J_\cI$ for the three different architectures as $J_{\cL, n}$, of $J_{\cI,n}$ with $n \in \{3, 4, 5\}$ respectively, corresponding to the 3 layer, 4 layer or 5 layer autoencoder. In this Section, we present various results about the structure of the eigenvalues of $J_{\cL, n}$ and $J_{\cI, n}$. We see that while the eigenvalues are not consistent with the autoencoders being projections maps onto a data manifold $M_\cD$, neither are they too far from said condition. 

First we observe that, while we are working over $\C$ the imaginary parts of most of these eigenvalues are small. Figure \ref{fig:angle} shows distribution of the arguments of the eigenvalues, measured in radians. Across all architecture types and latent dimension, the median argument of the eigenvalues remains small ($< 0.06$ radians). 

If the autoencoder locally defined projections, we would expect the eigenvalues to be real, clustering around $0$ and $1$.  Because we see some complex eigenvalues, we know that the autoencoder not only stretches but also rotates the tangent space.  The fact that these angles are usually small suggests that this rotation is not a significant part of the autoencoder's action, however.  Moreover, if one perturbs a matrix with real eigenvalues, some of which are equal or close to each other, one expects to see at most small angles due to the perturbation.  This is consistent with the empirical observations in Figure \ref{fig:angle}.

\begin{figure}
\begin{subfigure}[t]{0.3\textwidth}
\includegraphics[width=\textwidth]{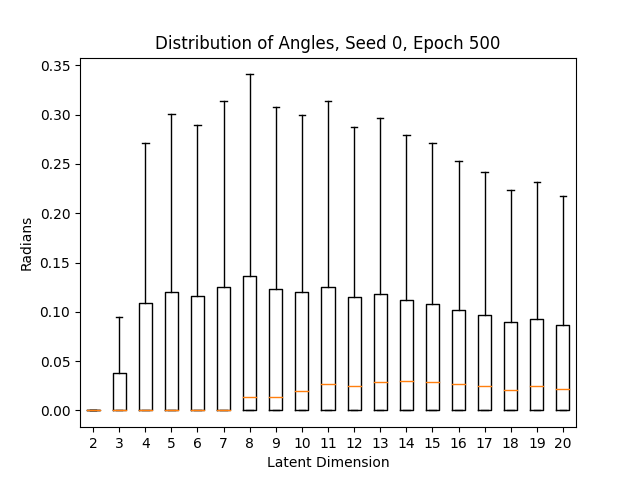}
\end{subfigure}
\begin{subfigure}[t]{0.3\textwidth}
\includegraphics[width=\textwidth]{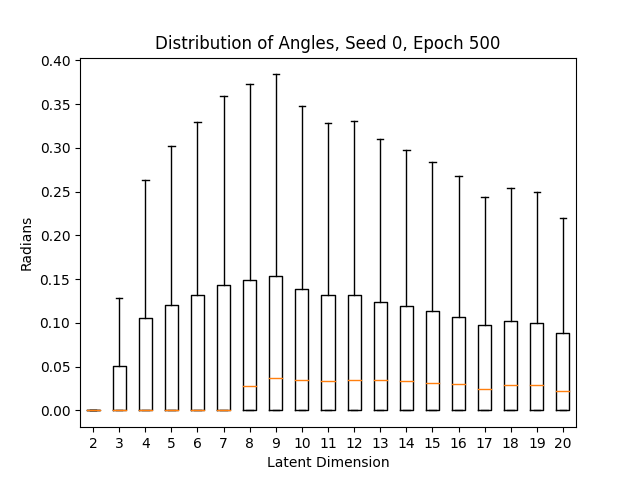}
\end{subfigure}
\begin{subfigure}[t]{0.3\textwidth}
\includegraphics[width=\textwidth]{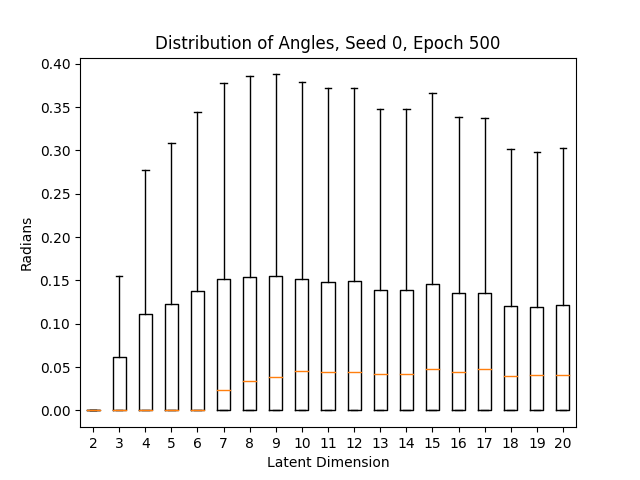}
\end{subfigure}
\newline \centering
\begin{subfigure}[t]{0.3\textwidth}
\includegraphics[width=\textwidth]{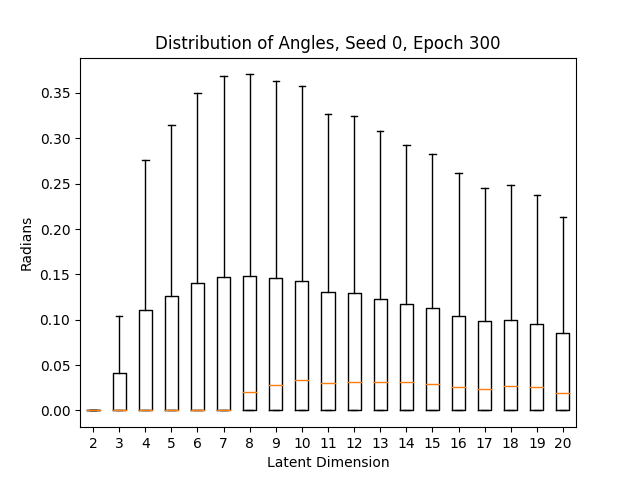}
\end{subfigure}
\caption{Distributions of the arguments of the eigenvalues of $J_{\cL, n}$ for $n \in \{3, 4,5\}$ (top). Distribution of the arguments of the eigenvalues of $J_{\cI, 4}$ (bottom). }
\label{fig:angle} \end{figure}

Next, we consider the algebraic and geometric means of the eigenvalues of $J_{\cL, n}$ for $n \in \{3, 4,5\}$. Note that when $\dim(\cL) = d$, the arithmetic mean of the eigenvalues as $\Tr(J_{\cI, n})/d$ and the geometric mean is $\omega_{\cL, n}^{1/d}$. Figure \ref{fig:arith} shows the median arithmetic means while Figure \ref{fig:geom} shows the geometric means. The second rows of these figures shows the breakdown of median means by class, while the third row compares the median means for the eigenvalues of $J_\cI$ and $J_\cL$. We see that the median arithmetic and geometric means are near 1 for every dimension and across architectures. After latent dimension 4, the median arithmetic and geometric means decrease with the latent dimension. This indicates that the latent dimensions 2 and 3 were likely too small for the information in the eigenvalues to overcome the noise of the necessary projection. It is curious to observe that at no latent dimension architecture or class are these median means above 1. In other words, while the autoencoder is close to projection near the training data as expected, it is also contracting in the directions of the eigenvectors. It is also worth noting that the arithmetic and geometric means are very similar to each other. This suggests that the eigenvalues of $J_\cI(x)$ are qualitatively close to these means, in that very few of them are either extremely large or extremely close to $0$.  If some of the eigenvalues were significantly larger or smaller than the rest, the arithmetic mean would be much larger than the geometric mean.

Looking at the breakdown of the median arithmetic and geometric means, we notice that the median means for class one in MNIST are significantly lower (but still close to one), while the other classes behave very similarly to each other. Comparing the median means of $\vec{\lambda}_\cI$ and $\vec{\lambda}_\cL$, we see that the median means (both arithmetic and geometric) of the eigenvalues of $J_\cI$ are larger than the median means of the eigenvalues of $J_\cL$. 

\begin{figure}
\begin{subfigure}[t]{0.3\textwidth}
\includegraphics[width=\textwidth]{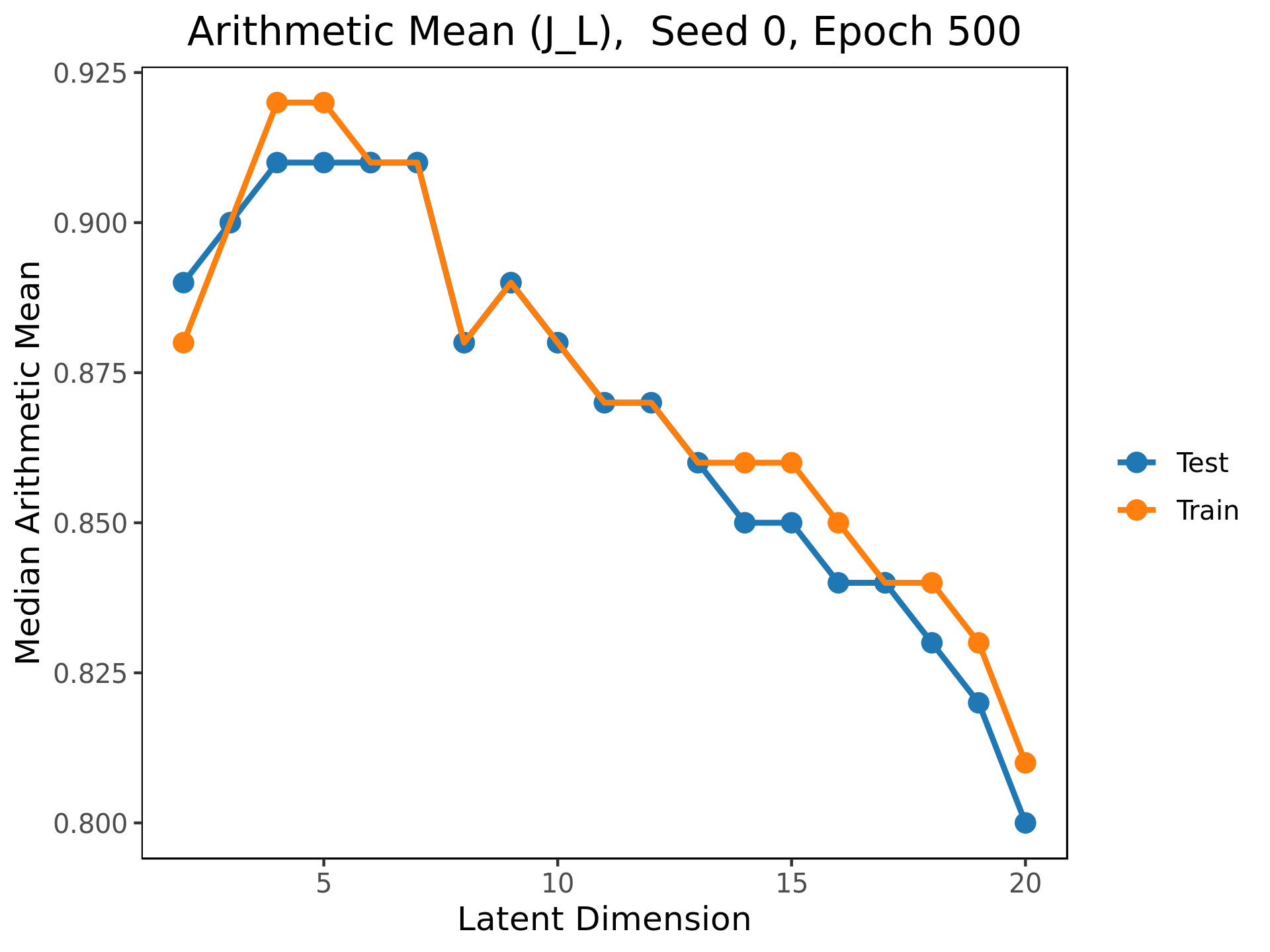}
\end{subfigure}
\begin{subfigure}[t]{0.3\textwidth}
\includegraphics[width=\textwidth]{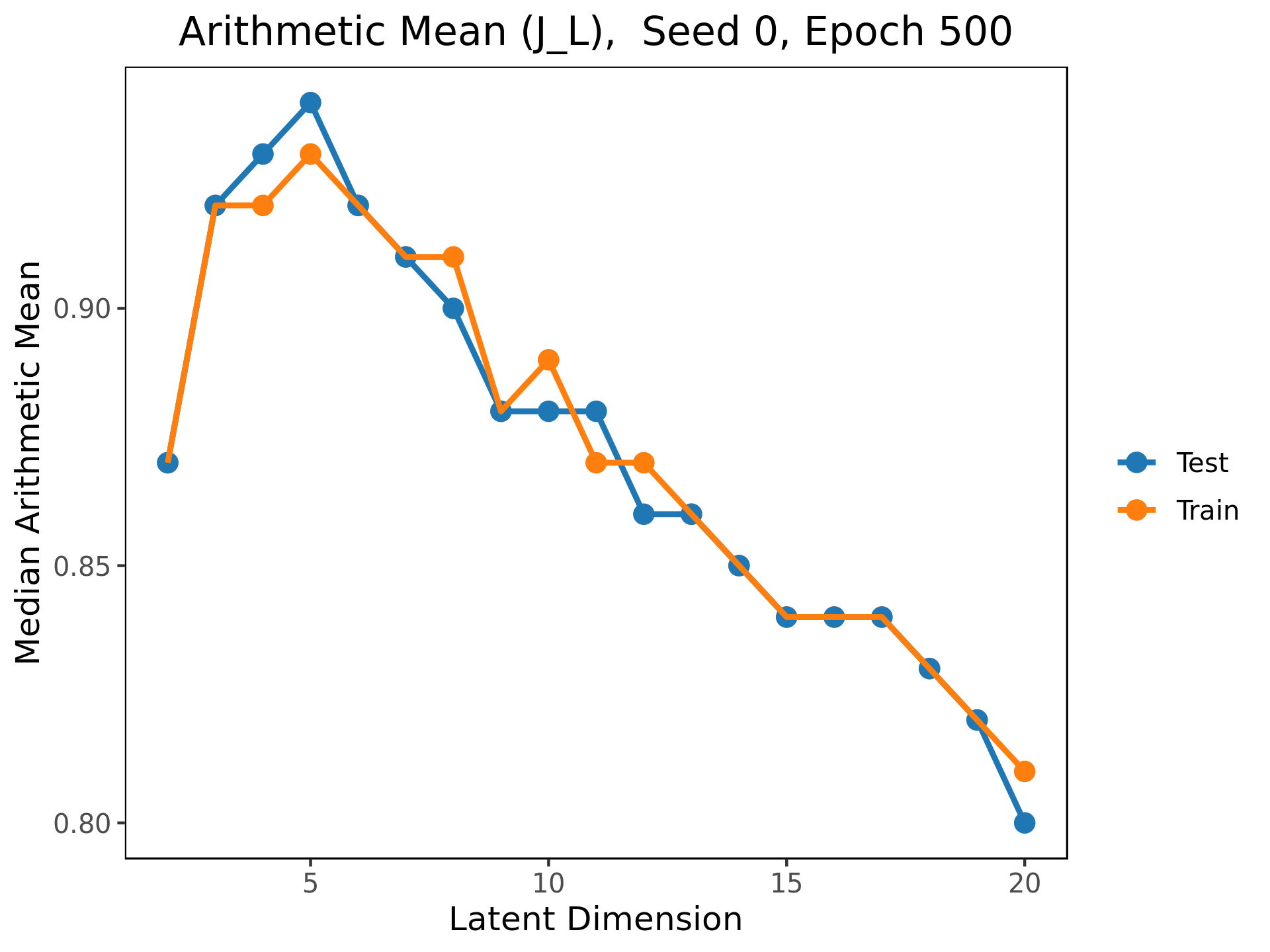}
\end{subfigure}
\begin{subfigure}[t]{0.3\textwidth}
\includegraphics[width=\textwidth]{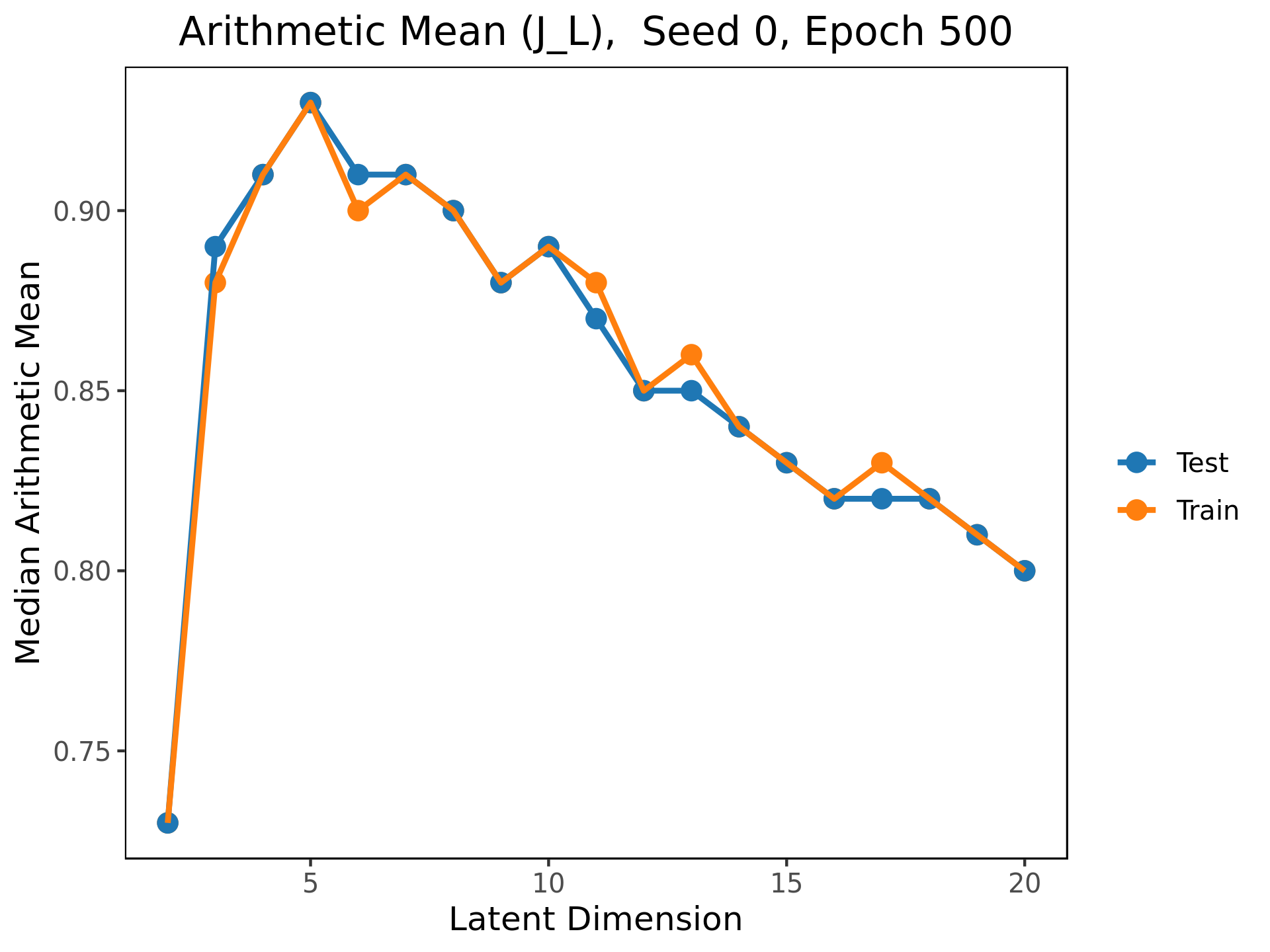}
\end{subfigure}
\newline 
\begin{subfigure}[t]{0.3\textwidth}
\includegraphics[width=\textwidth]{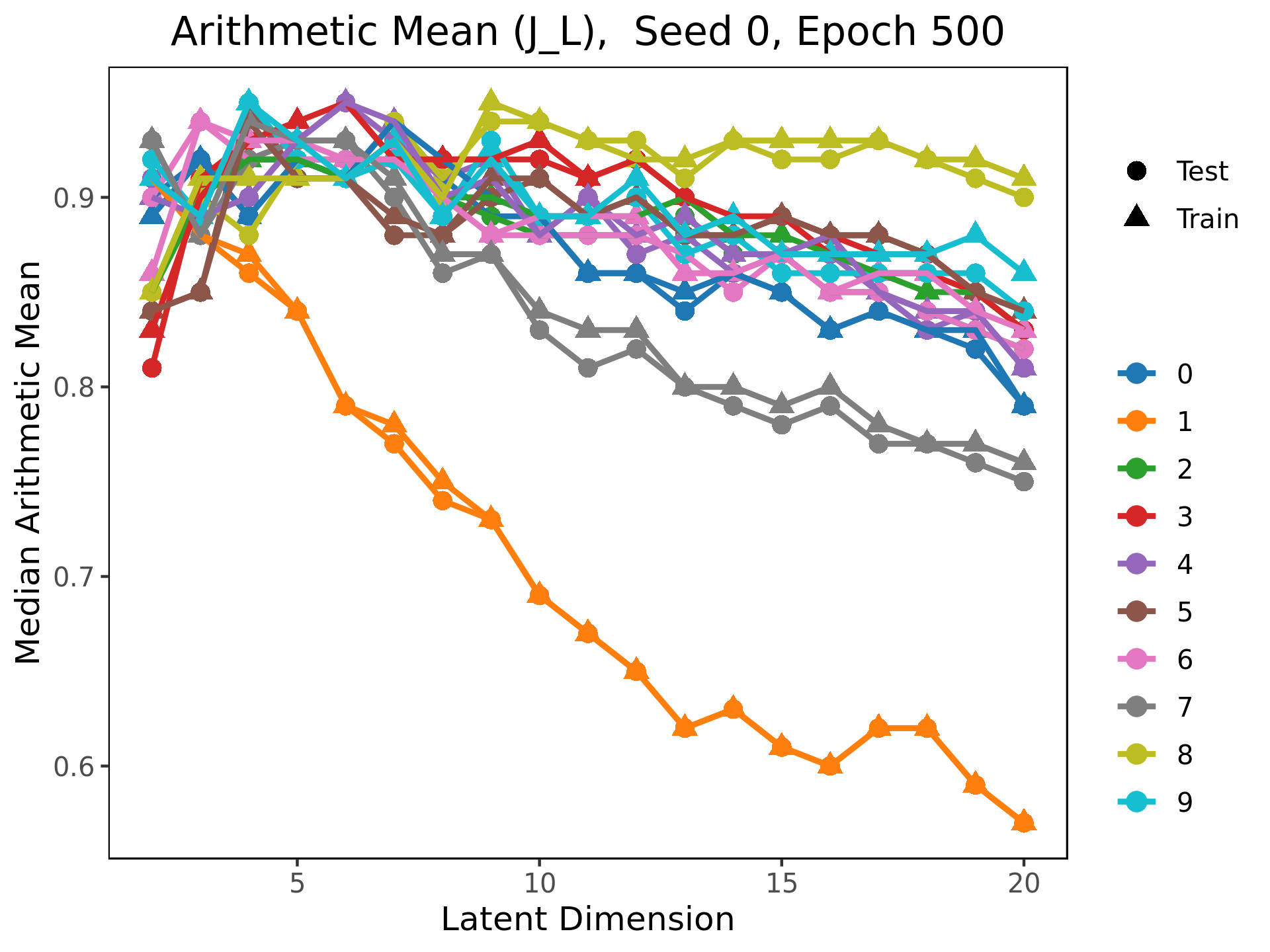}
\end{subfigure}
\begin{subfigure}[t]{0.3\textwidth}
\includegraphics[width=\textwidth]{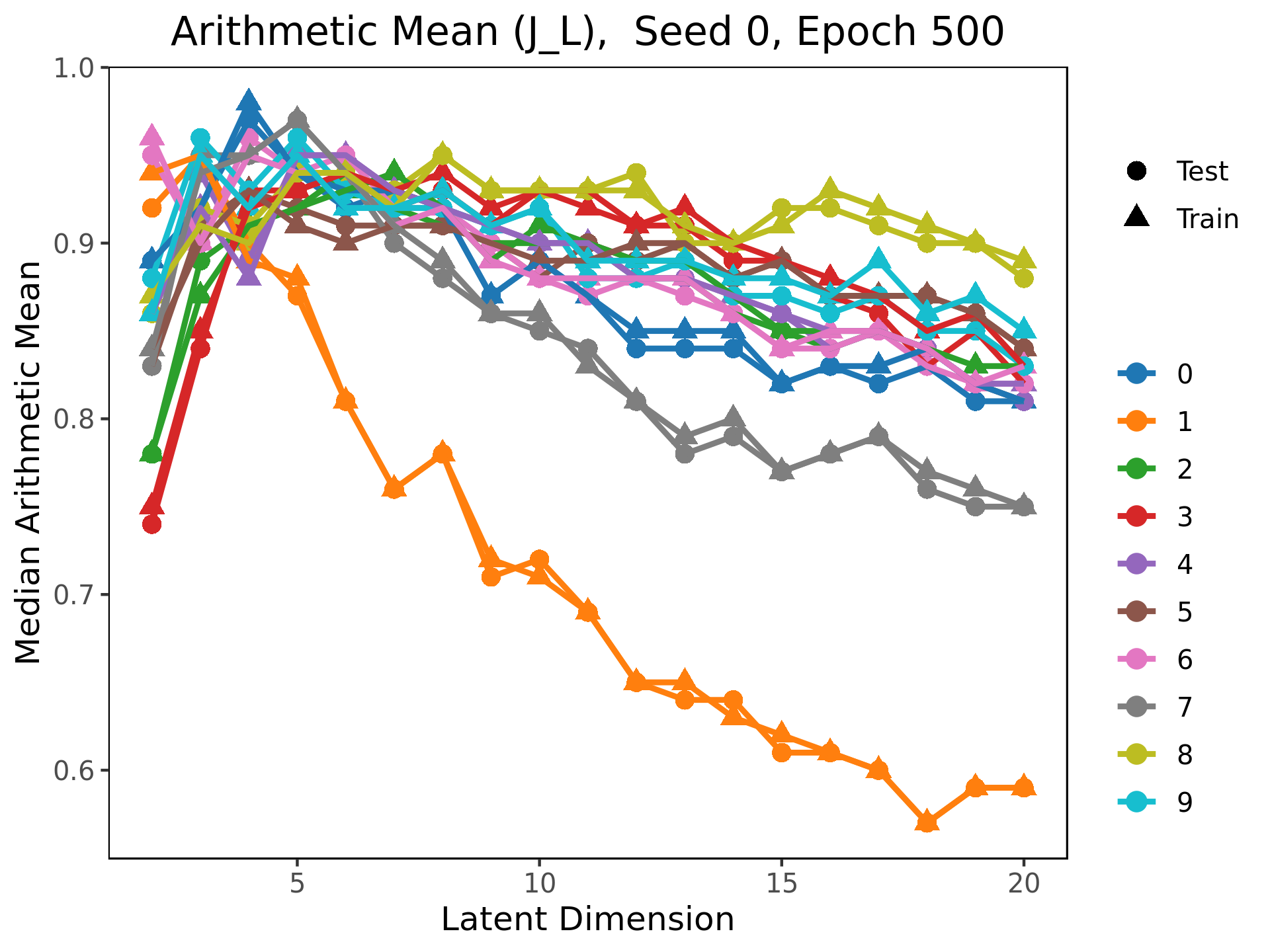}
\end{subfigure}
\begin{subfigure}[t]{0.3\textwidth}
\includegraphics[width=\textwidth]{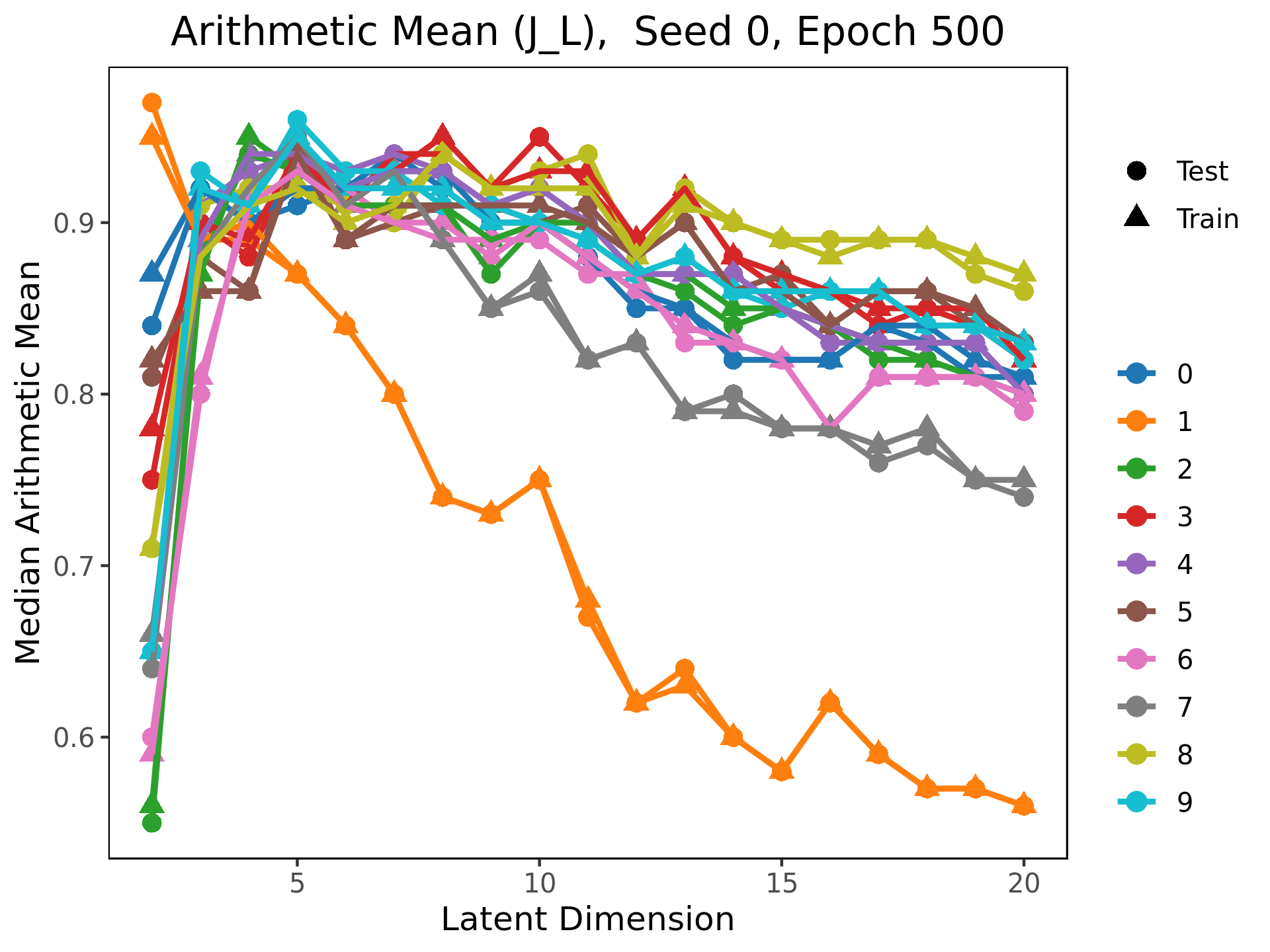}
\end{subfigure}
\newline 
\centering
\begin{subfigure}[t]{0.3\textwidth}
\includegraphics[width=\textwidth]{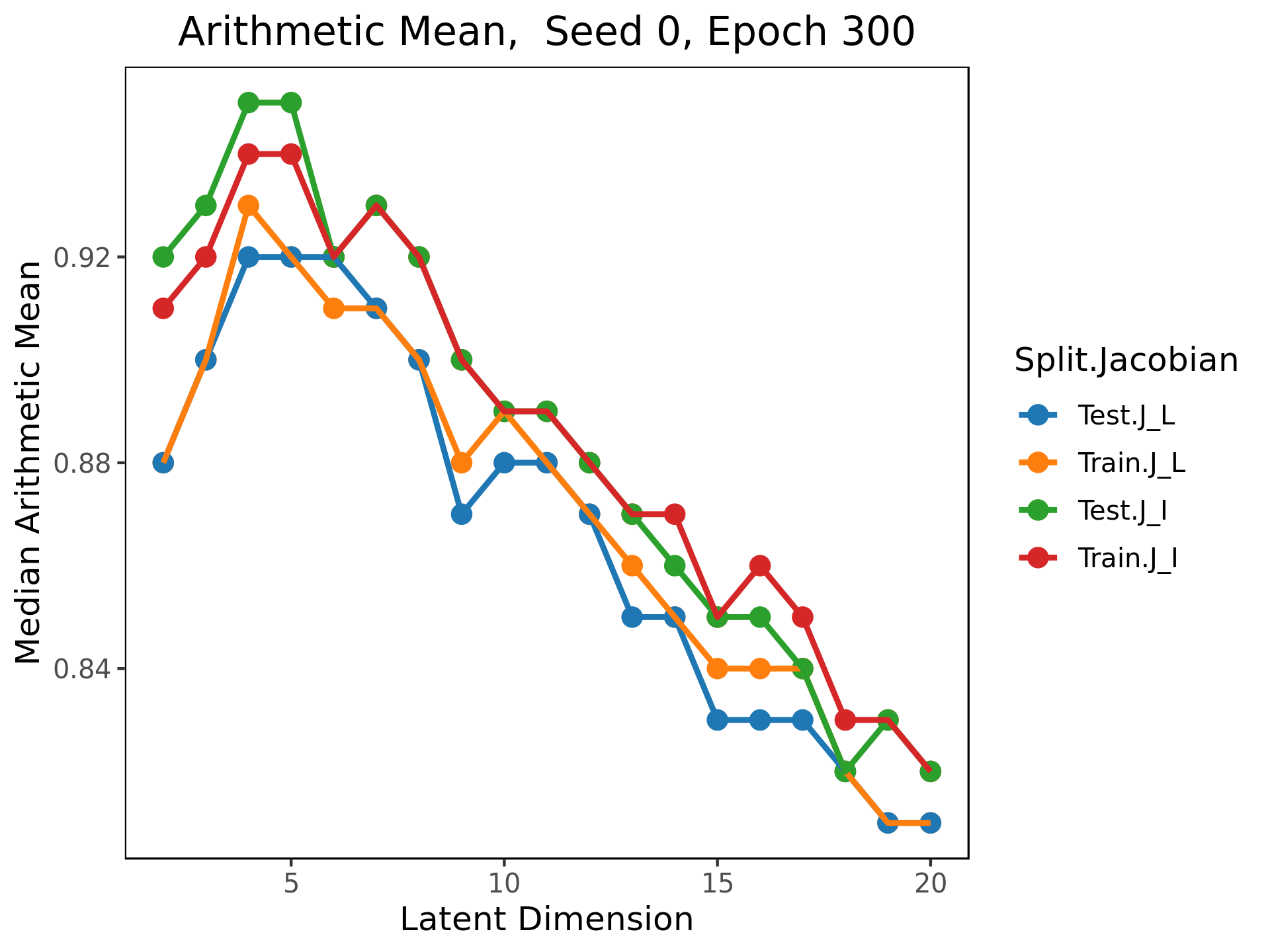}
\end{subfigure}

\caption{The arithmetic means of the eigenvalues start out close to one at low latent dimension, and decrease as latent dimension increases. This is true when the data is broken down by class as well. Note that at high latent dimension, the arithmetic means for class 1 is much lower than the rest of the classes. The median arithmetic means of $\vec{\lambda}_\cI$ tend to be higher than the median arithmetic means of $\vec{\lambda}_\cL$.}\label{fig:arith}
\end{figure}

\begin{figure}
\begin{subfigure}[t]{0.3\textwidth}
\includegraphics[width=\textwidth]{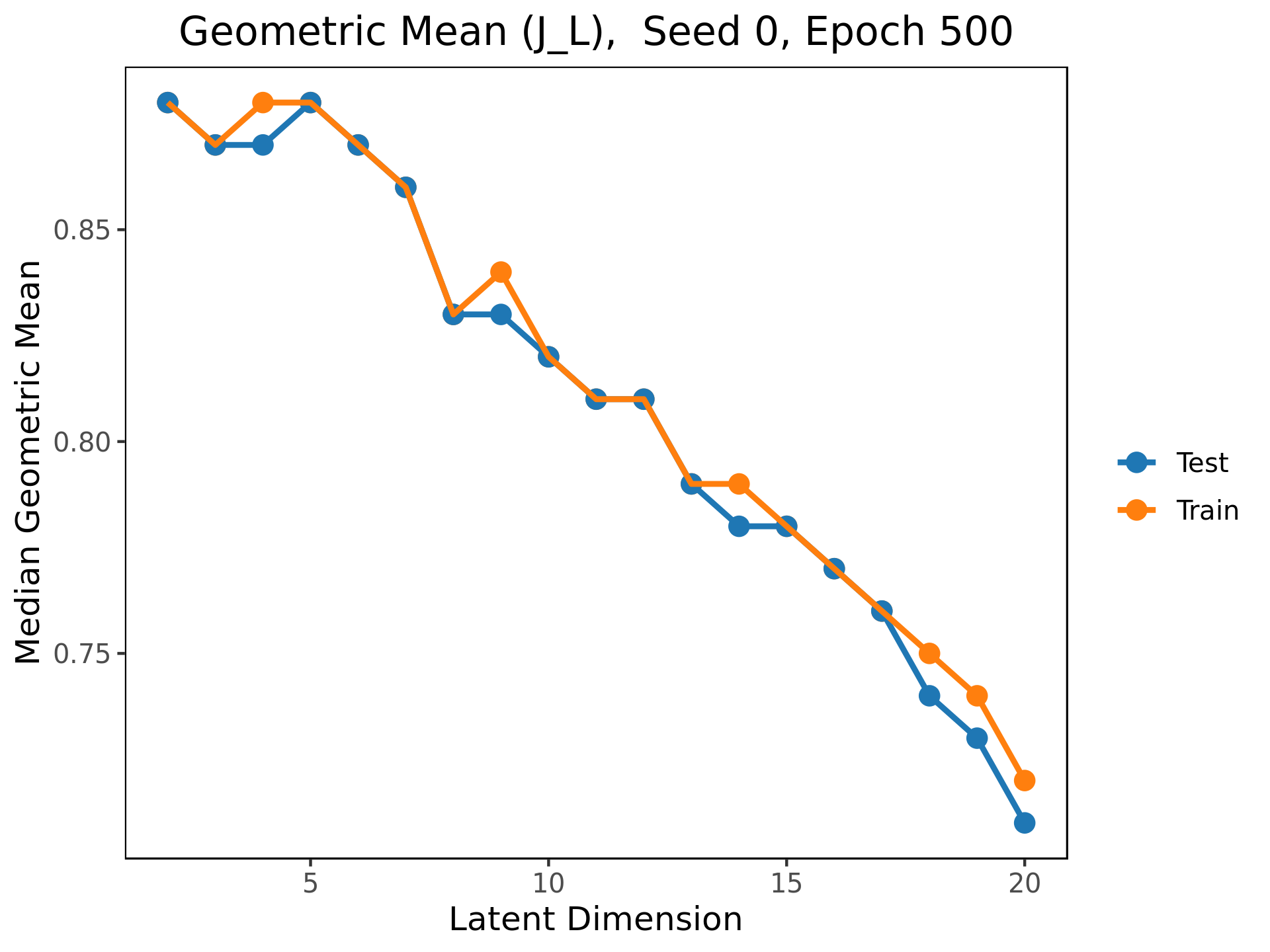}
\end{subfigure}
\begin{subfigure}[t]{0.3\textwidth}
\includegraphics[width=\textwidth]{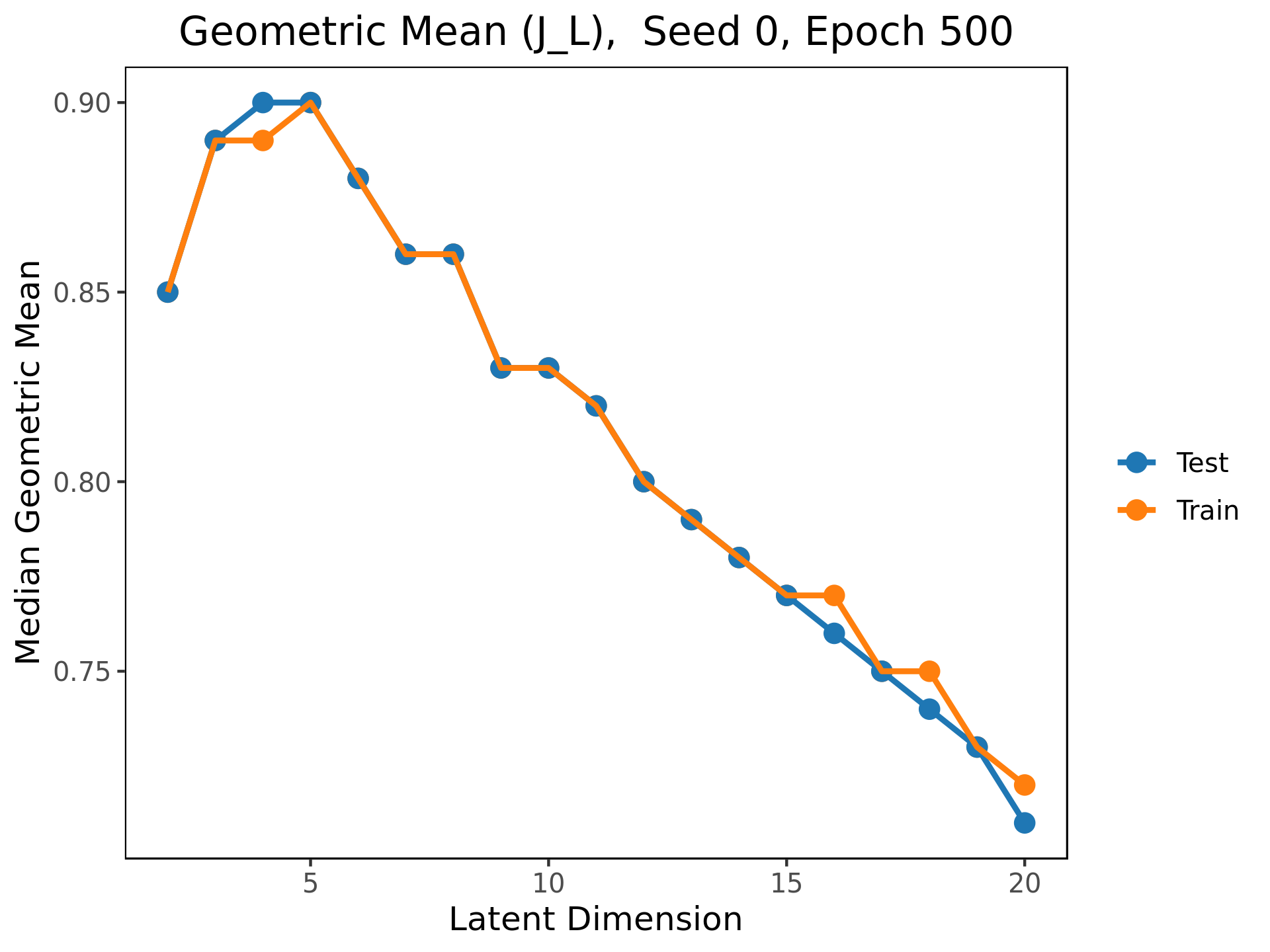}
\end{subfigure}
\begin{subfigure}[t]{0.3\textwidth}
\includegraphics[width=\textwidth]{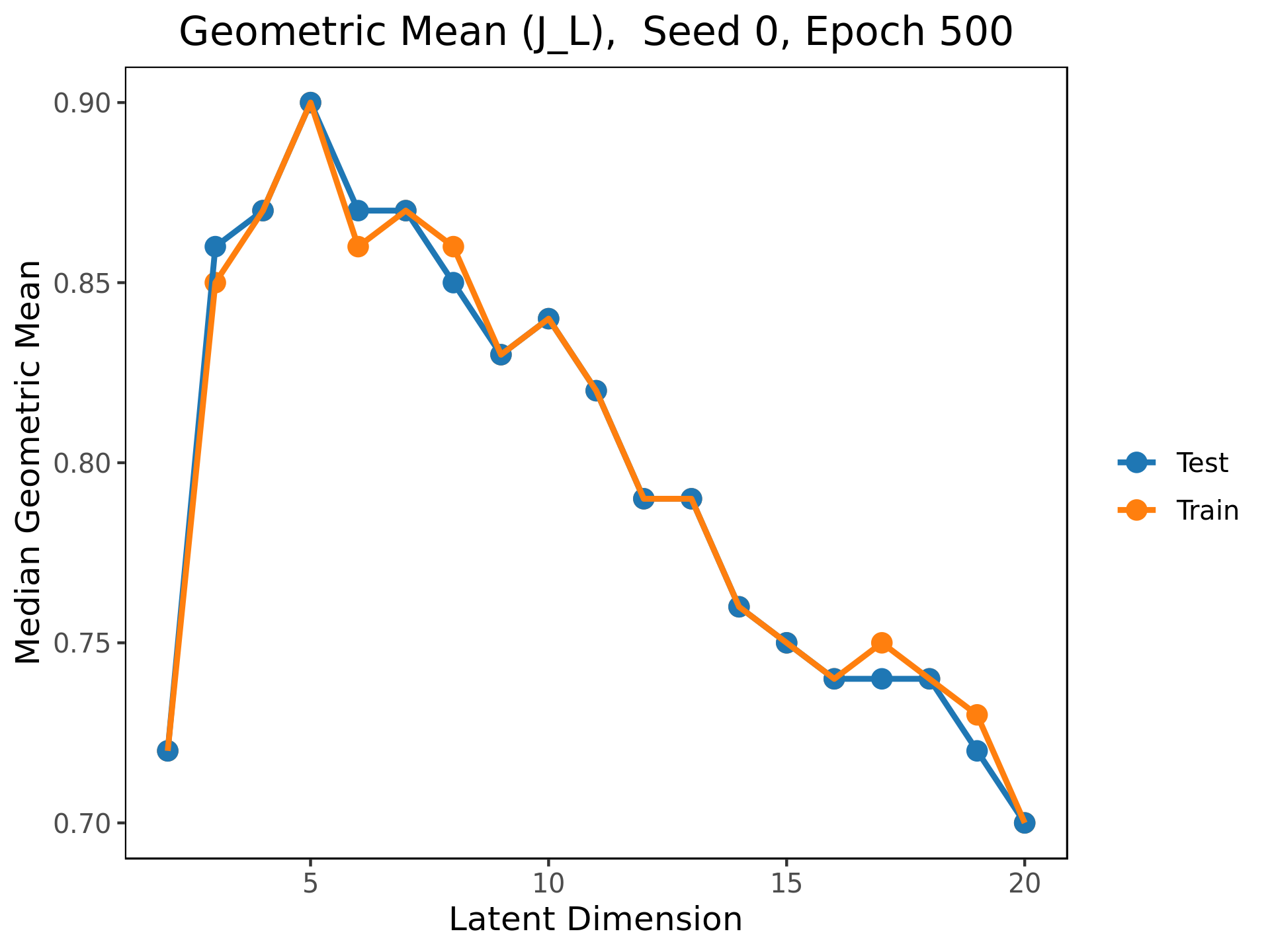}
\end{subfigure}
\newline 
\begin{subfigure}[t]{0.3\textwidth}
\includegraphics[width=\textwidth]{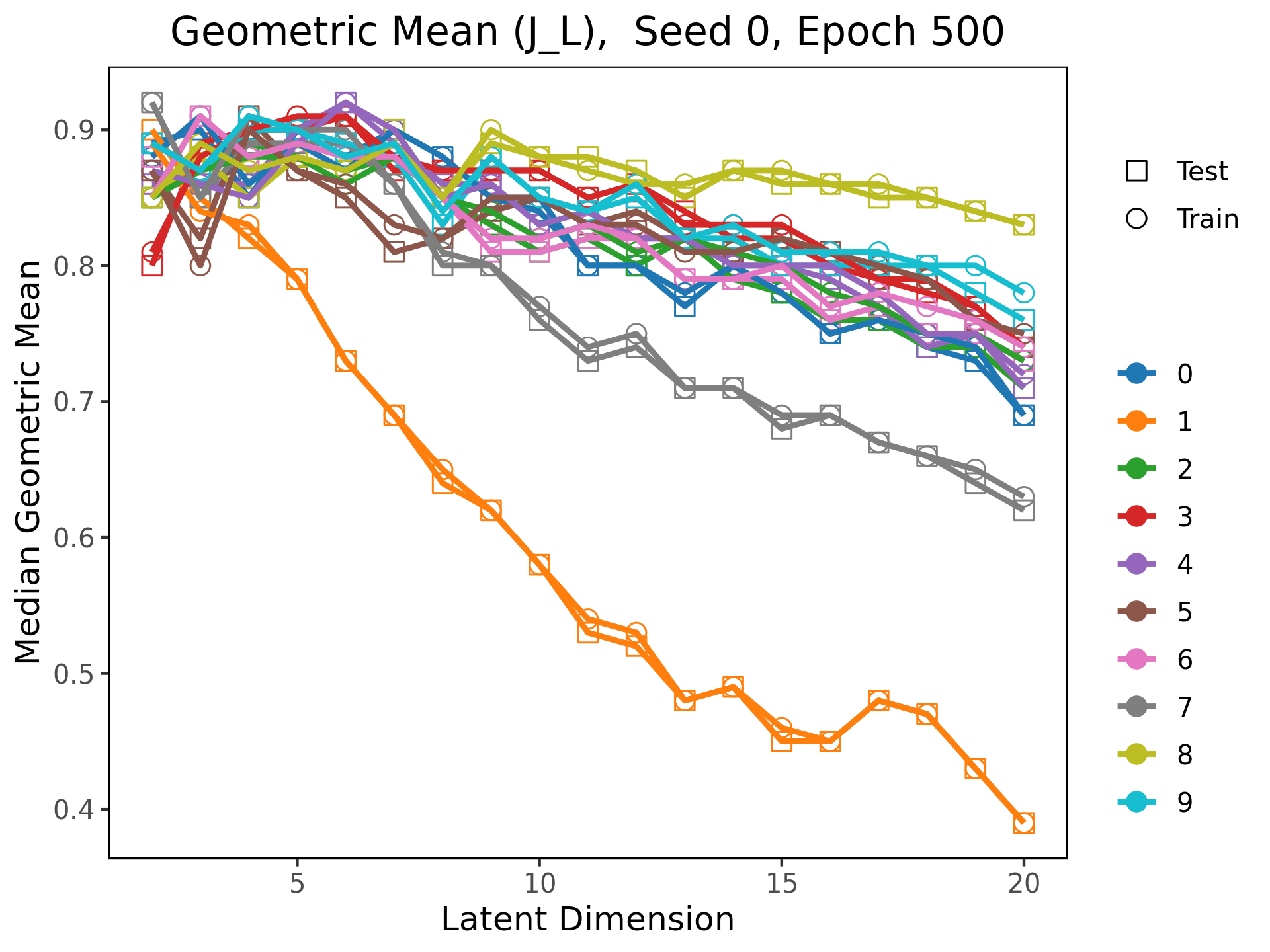}
\end{subfigure}
\begin{subfigure}[t]{0.3\textwidth}
\includegraphics[width=\textwidth]{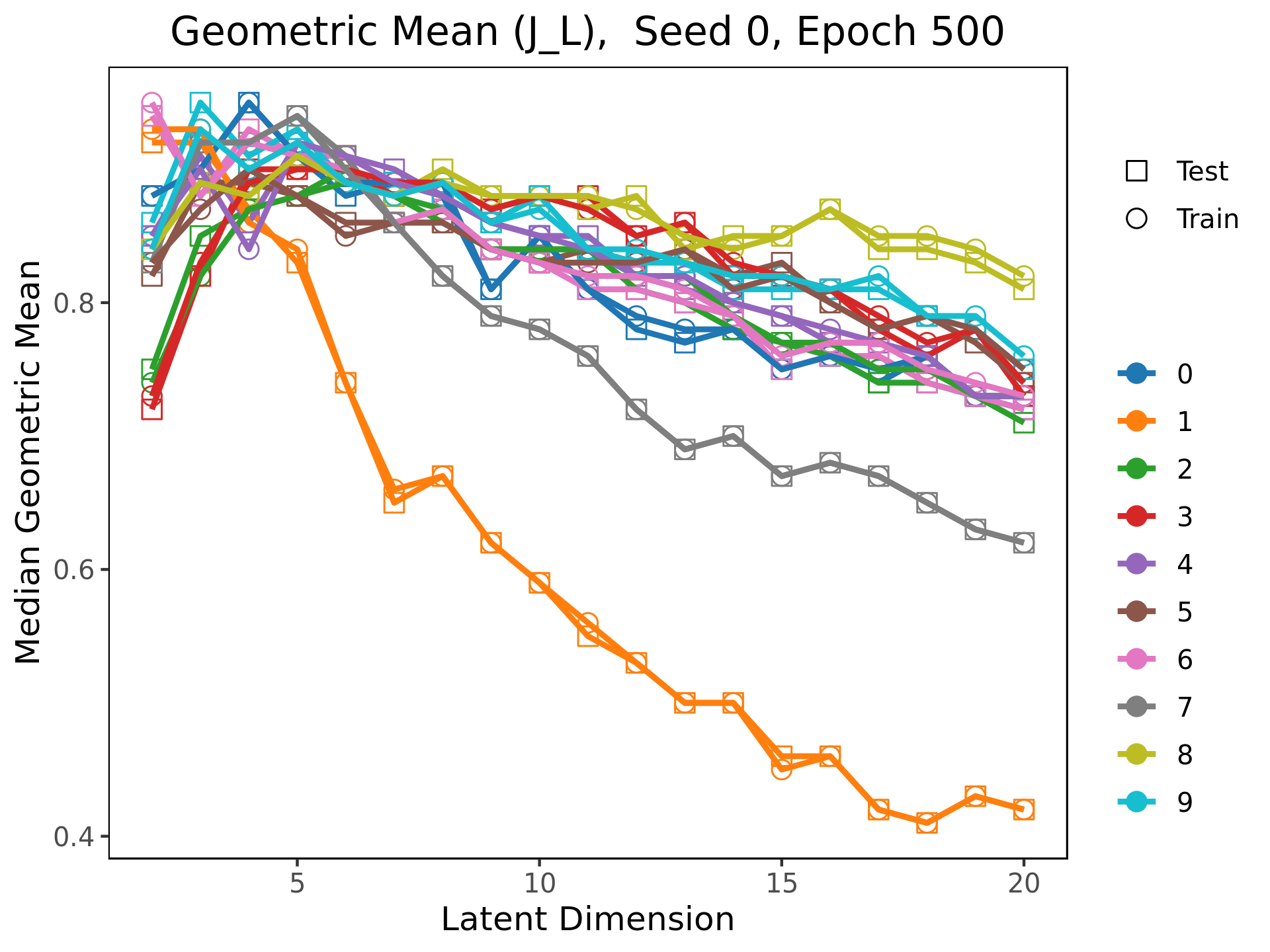}
\end{subfigure}
\begin{subfigure}[t]{0.3\textwidth}
\includegraphics[width=\textwidth]{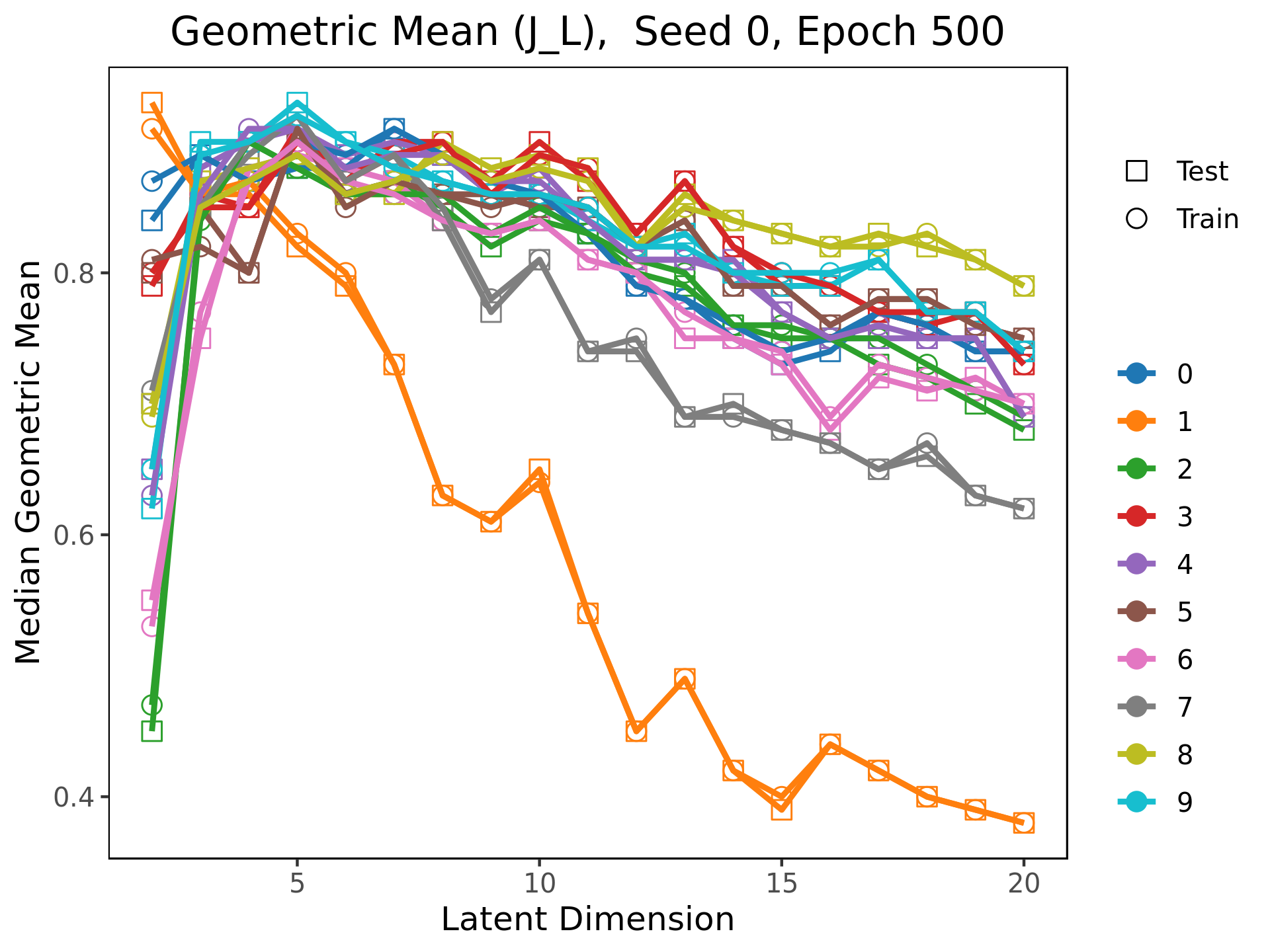}
\end{subfigure}
\newline
\centering
\begin{subfigure}[t]{0.3\textwidth}
\includegraphics[width=\textwidth]{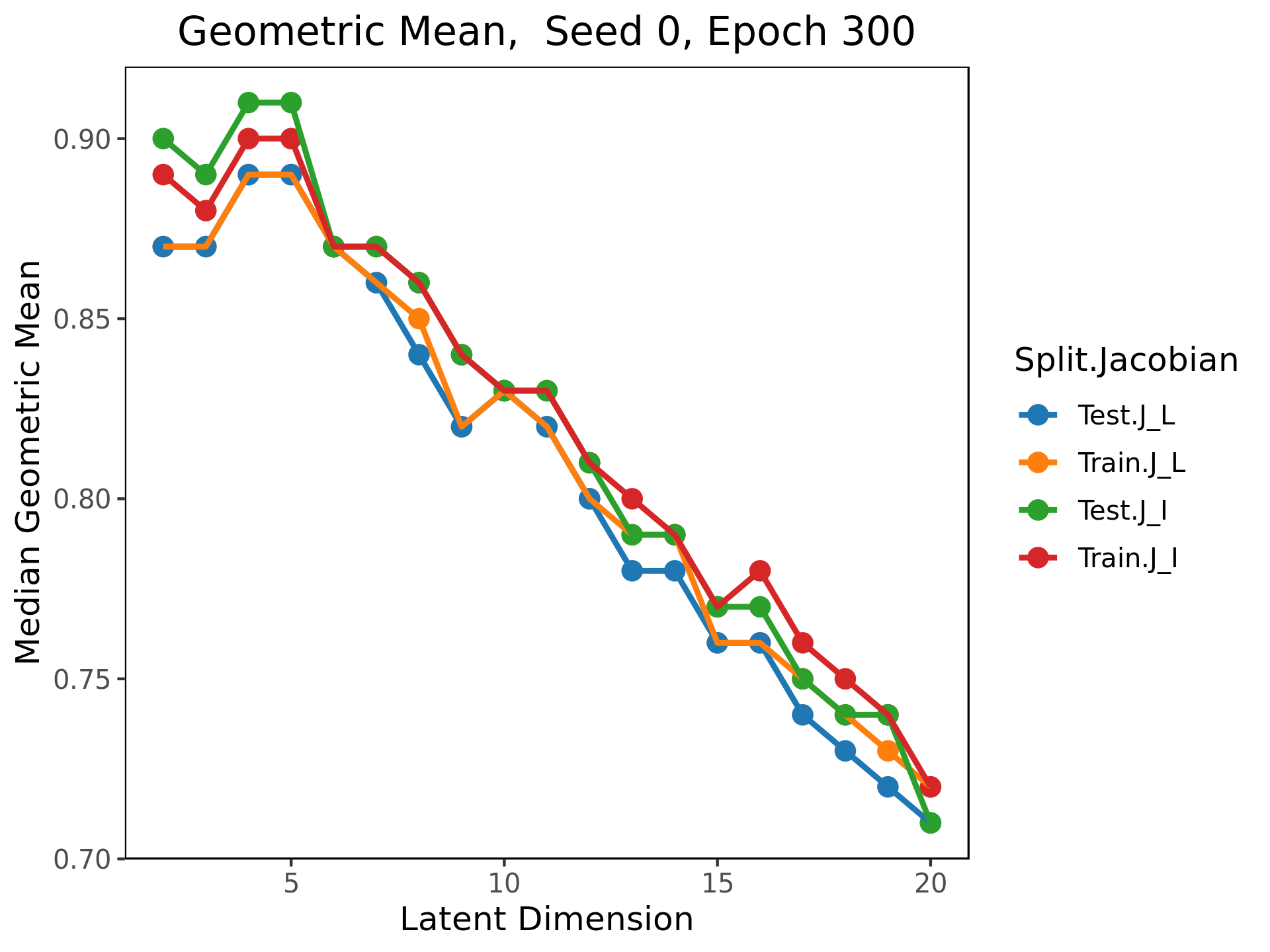}
\end{subfigure}

\caption{The geometric means of the eigenvalues start out close to one at low latent dimension, and decrease as latent dimension increases. This is true when the data is broken down by class as well. Note that at high latent dimension, the geometric means for class 1 is much lower than the rest of the classes.The median geometric means of $\vec{\lambda}_\cI$ tend to be higher than the median geometric means of $\vec{\lambda}_\cL$. }\label{fig:geom}
\end{figure}

The fact that the median arithmetic mean of the eigenvalues is less than one implies that the autoencoder cannot locally be the identity map onto its image.  Despite this, the empirics show that it is not qualitatively too far from this.  In particular, we see that very few of the eigenvalues are close to $0$, a pattern which persists across the architecture, latent dimension, and class. (see Figure \ref{fig:zeroes}).

\begin{figure}
\begin{subfigure}[t]{0.3\textwidth}
\includegraphics[width=\textwidth]{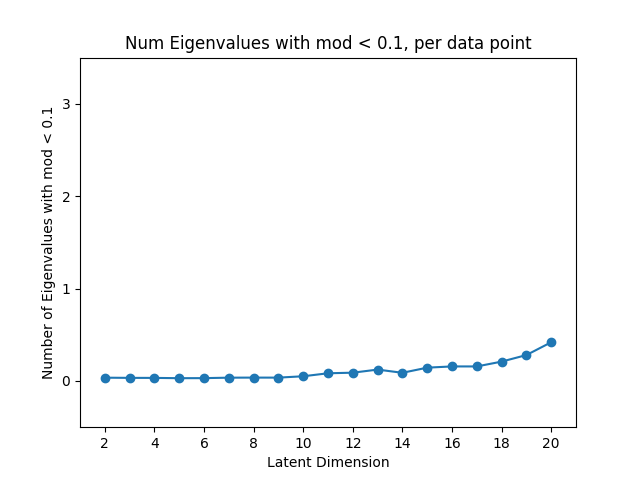}
\end{subfigure}
\begin{subfigure}[t]{0.3\textwidth}
\includegraphics[width=\textwidth]{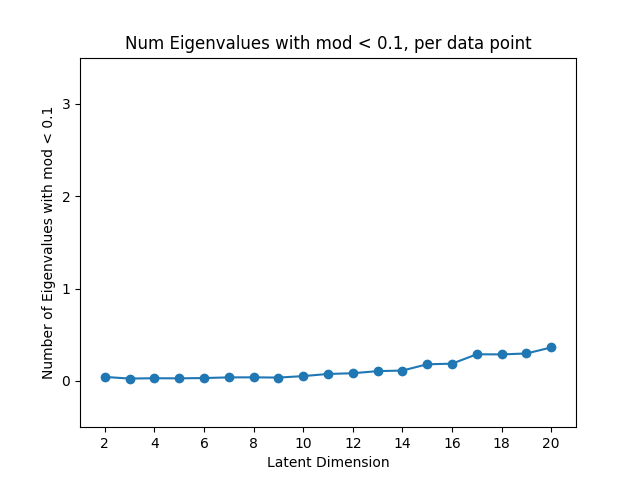}
\end{subfigure}
\begin{subfigure}[t]{0.3\textwidth}
\includegraphics[width=\textwidth]{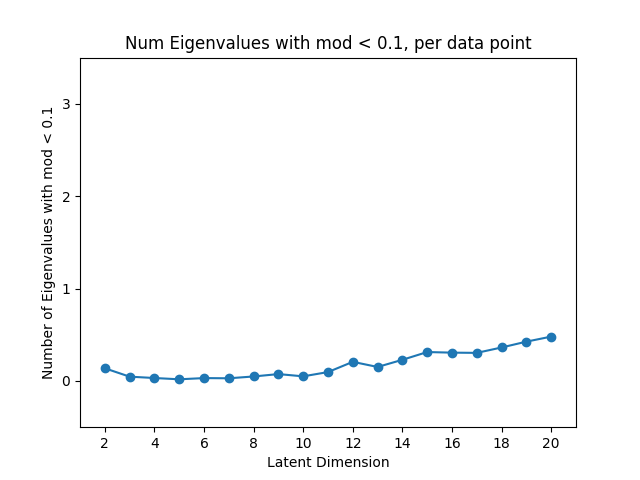}
\end{subfigure}
\newline 
\begin{subfigure}[t]{0.3\textwidth}
\includegraphics[width=\textwidth]{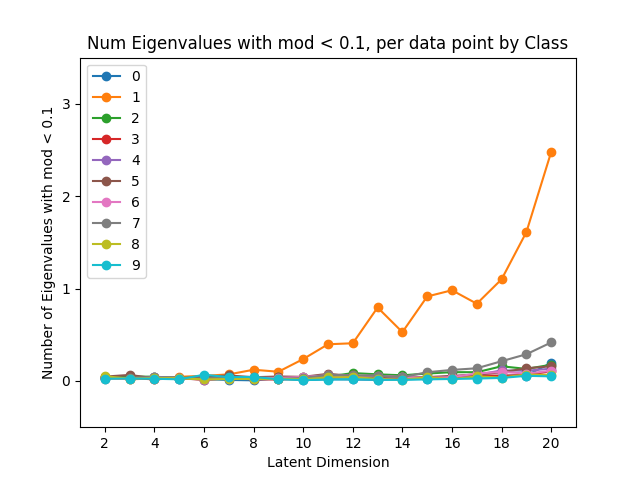}
\end{subfigure}
\begin{subfigure}[t]{0.3\textwidth}
\includegraphics[width=\textwidth]{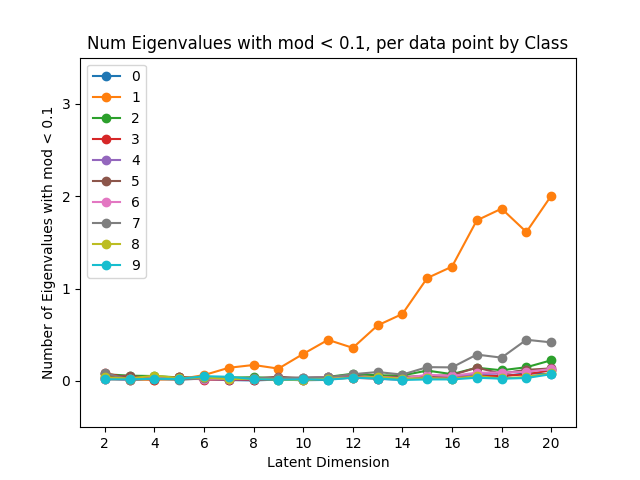}
\end{subfigure}
\begin{subfigure}[t]{0.3\textwidth}
\includegraphics[width=\textwidth]{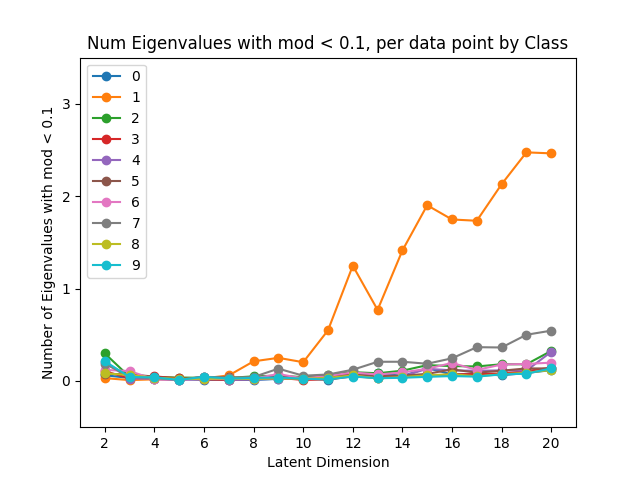}
\end{subfigure}

\caption{The number of eigenvalues which have absolute value less than $0.1$, for $J_{\cL,n}$, $n\in\{3,4,5\}$ (top), and the number of such eigenvalues, broken down by class, for $J_{\cL,n}$, $n\in\{3,4,5\}$ (bottom).  We observe that very few eigenvalues have absolute value less than $0.1$, except in class 1.  Even in class 1, the number of eigenvalues less than $0.1$ is at most $2$, in latent dimension $20$.}
\label{fig:zeroes}
\end{figure}

Finally, we see see that $J_\cL$ is orientation reversing for a very small fraction of the data points. That is, $\omega_*(x)<0$. Figure \ref{fig:orientation} shows the proportion of test and training data points that are orientation reversing vary from $\approx 6\%$ at low latent dimension to $< 1\%$ at higher latent dimension. Furthermore, we see a similar pattern for the product of the top $d$ eigenvalues of $J_\cI(x)$, though in general, there rate of orientation reversals for $J_\cI$ is slightly higher than that for $J_\cL$. This is another qualitative sense in which $J_\cL,J_\cI$ resemble the identity or a projection (respectively), as those maps are never orientation reversing.

\begin{figure}
\begin{subfigure}[t]{0.3\textwidth}
\includegraphics[width=\textwidth]{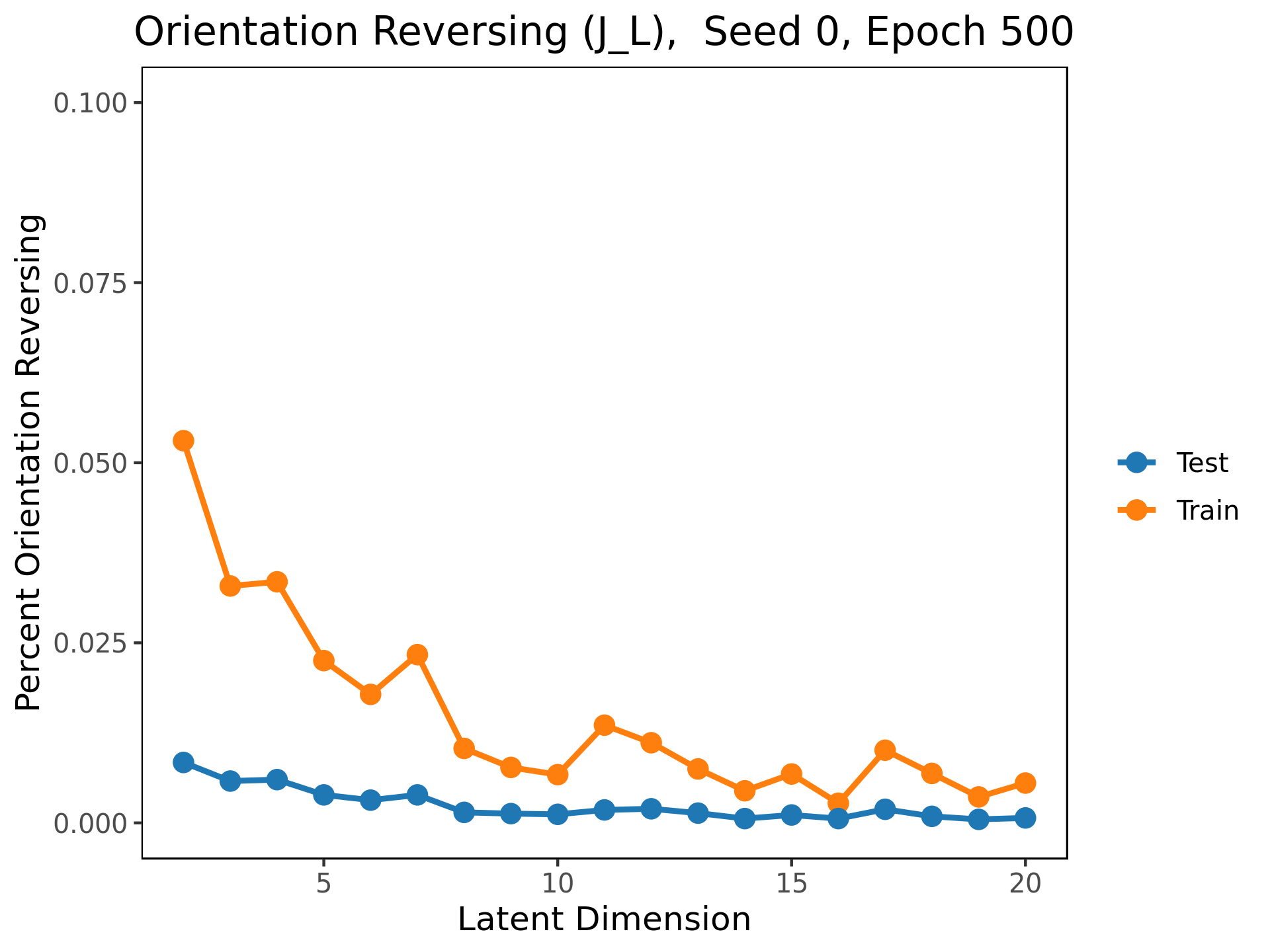}
\end{subfigure}
\begin{subfigure}[t]{0.3\textwidth}
\includegraphics[width=\textwidth]{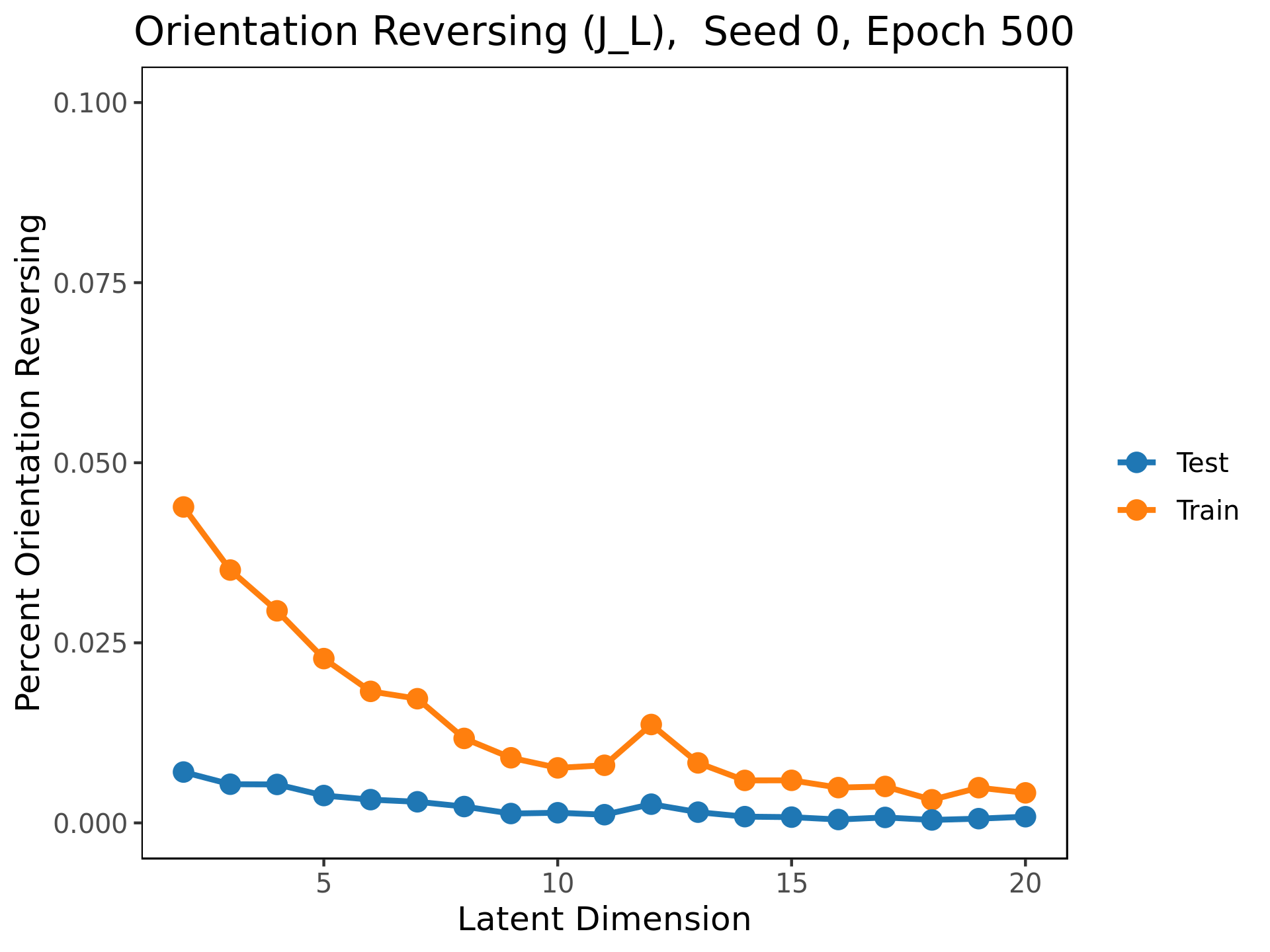}
\end{subfigure}
\begin{subfigure}[t]{0.3\textwidth}
\includegraphics[width=\textwidth]{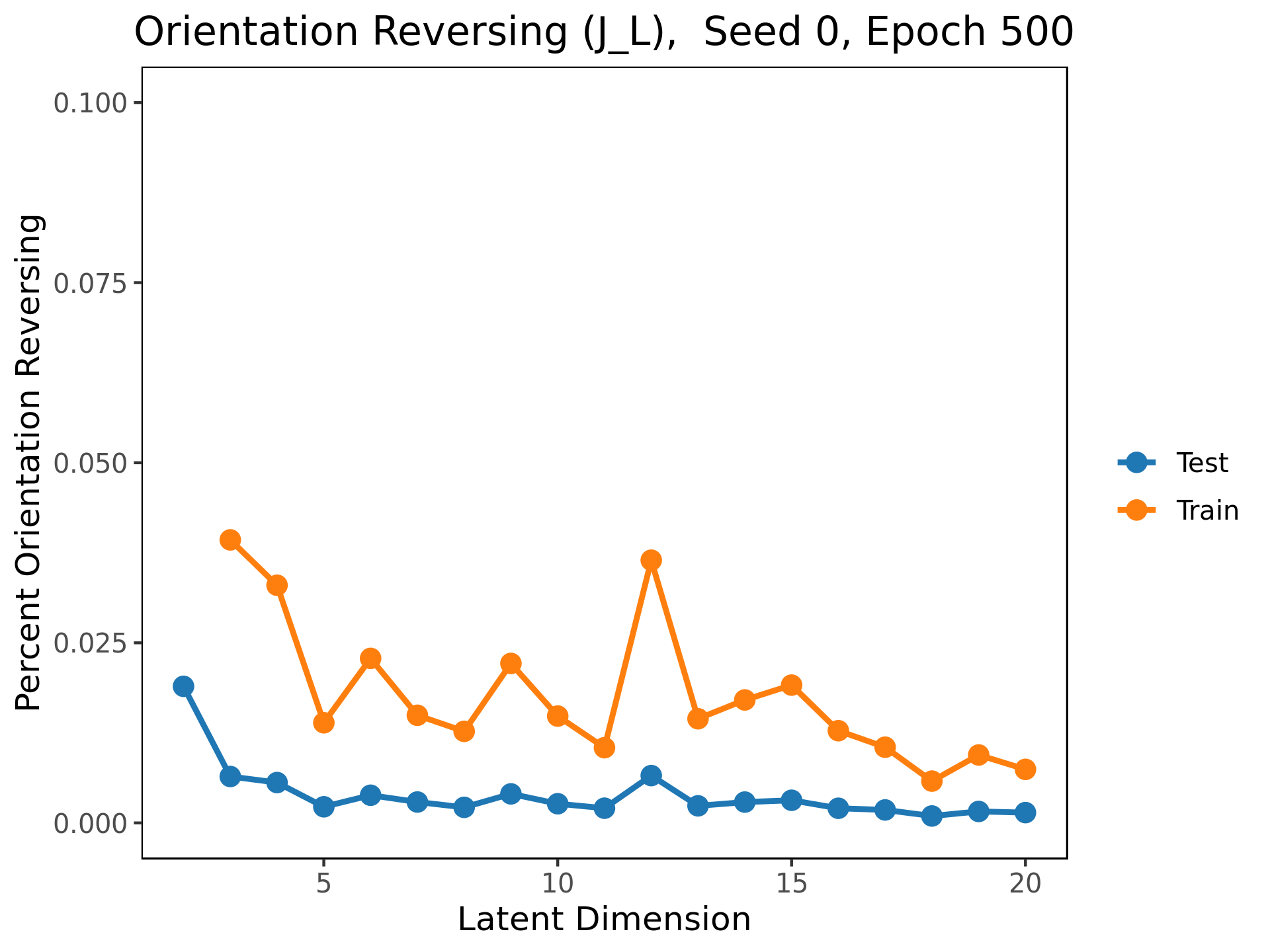}
\end{subfigure}
\newline
\centering
\begin{subfigure}[t]{0.3\textwidth}
\includegraphics[width=\textwidth]{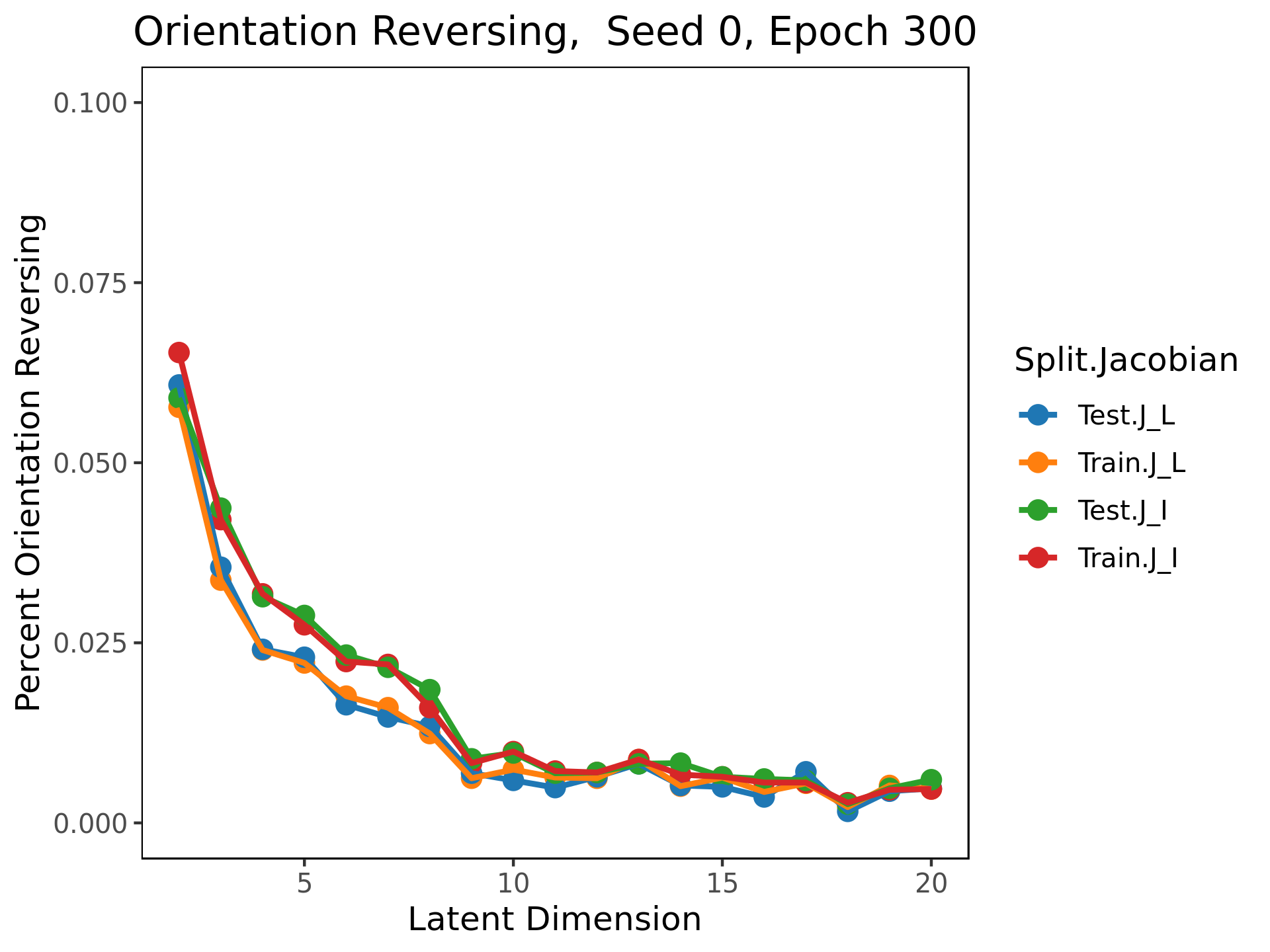}
\end{subfigure}

\caption{Proportion of points in $\cD$ such that $\omega_*(x)<0$.}\label{fig:orientation}
\end{figure}

\subsection{Predicting reconstruction error \label{sec:predictions}}

Next, we consider the predictive value of quantities $\log(\omega_{\cL, n})$ and $\log(\omega_{\cI, n})$ on the MSE of the test points. We model this with the linear regression of the form \ba MSE = \beta_0 + \beta_1 \log(\omega_*) + \beta_2 test\_indicator +  \beta_3 test\_indicator \times \log(\omega_*) \;,  \label{eq:regression} \ea where $\omega_* \in \{\omega_{\cI,n} , \omega_{\cL,n}\}$ and \bas test\_indicator(x) = \begin{cases} 1 & x \in \cD_{test} \\ 0 & x \in \cD_{train} \end{cases}\;.\eas The reason for this more complicated form of the regression is that the coefficient $\beta_3$ gives the excess change in MSE \emph{if} the point is a test point. When this coefficient is positive, we learn that increasing the independent variable ($\log(\omega_*)$) predicts a greater increase in MSE on test points than on traiing points. This would indicate that  the independent variable is a good predictor of whether the learned model will generalize well on a given new data point. 

In order to account for differences in the scale of the MSE and the volume forms acrosss different dimensions, all regressions are taken with respect to scaled variables. 

Note that, for $x \in \cD_{train}$, this reduces to \bas MSE = \beta_0 + \beta_1 \log(\omega_*) \;.\eas However, on test points, the linear coefficient of this regression is given by $\beta_1 + \beta_3$. The coefficients for this regressions in the four layer experiment is given in Table \ref{table:defaultregression}. Figure \ref{fig:log_pred} shows the slopes of this regression for each experiment across latent dimension. Note from Table \ref{table:defaultregression} that the intercepts for these regressions remain near $0$. 

Further note that  as latent dimension increases, the coefficient $\beta_3$ is consistently positive. That is, the same increase in $\ln \omega_*$ predicts a larger increase in expected reconstruction loss if the point is a test point, rather than a training point. In other words, observing larger $\log \omega_*$ on points that were not in the training set means that the corresponding reconstruction error on these points will be even greater than on the training points. I.e. the larger $\log(\omega_*)$ on test points, the higher the expected probability that the network has not properly generalized to that test point. When we break this down by class, we see the same pattern persist.


\begin{table}[h]
\tiny
\centering
\resizebox{\textwidth}{!}{
\csvreader[
	tabular = |l||r|r|r|r|r||r|r|r|r|r|,
	table head = \hline \bfseries{dim} & \multicolumn{5}{|c||}{{\bf Coefficients for the latent space }$\cL$} & \multicolumn{5}{|c|}{{\bf Coefficients for the input space }$\cI$}\\\hline,
	late after last line=\\\hline

]{figs_and_tabs/graphs_seed-0_epochs-300/exp-default/coef_df_Log_VF_combined2.csv}{}{\csvlinetotablerow}
}
\caption{Coefficients of the linear regression in equation \eqref{eq:regression} for $\omega_\cI$ and $\omega_\cL$ across latent dimensions. Run on data from seed 0, exp-default and 300 epochs.}
\label{table:defaultregression}
\end{table}

\begin{figure}
\begin{subfigure}[t]{0.3\textwidth}
\includegraphics[width=\textwidth]{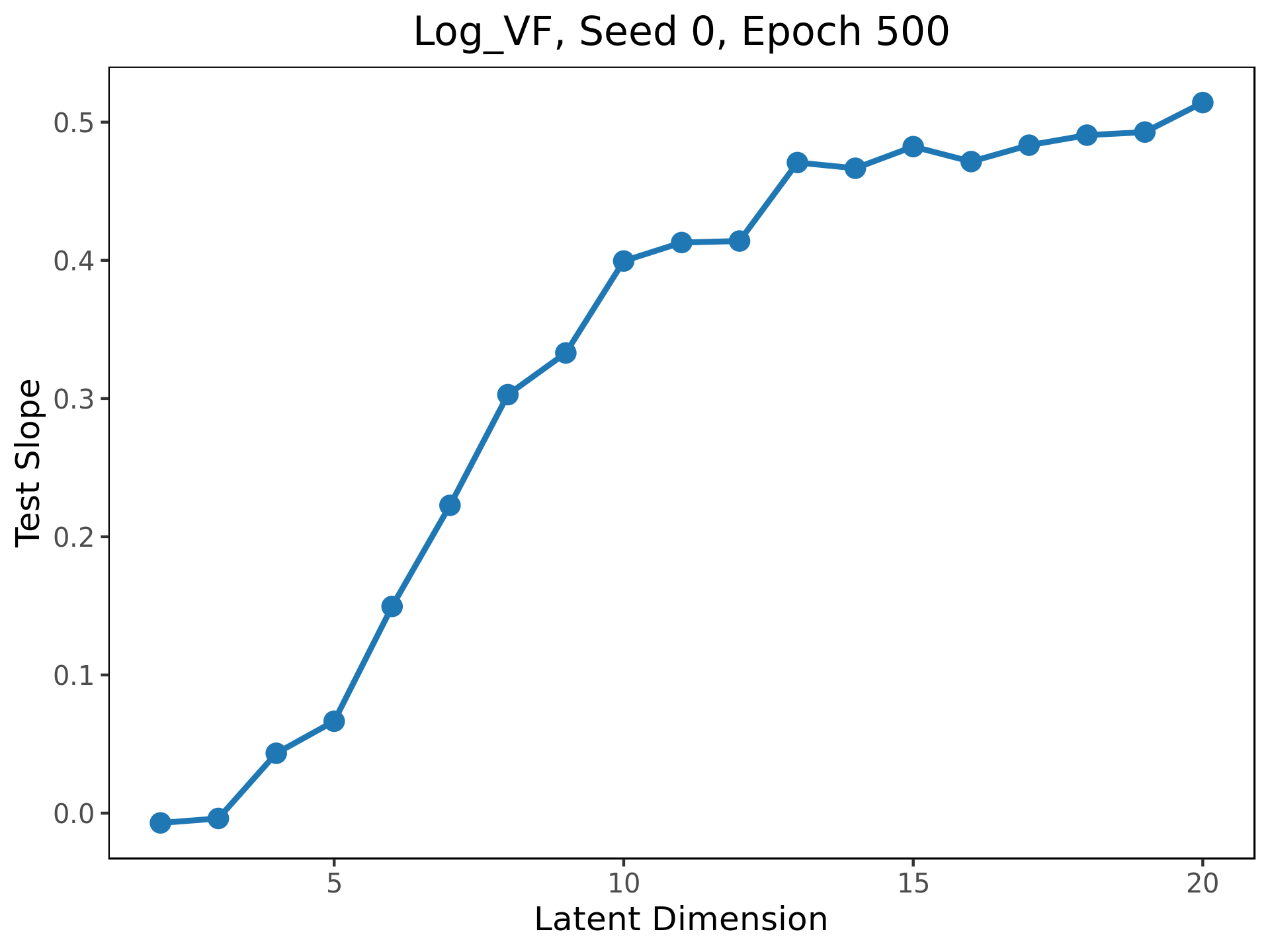}
\end{subfigure}
\begin{subfigure}[t]{0.3\textwidth}
\includegraphics[width=\textwidth]{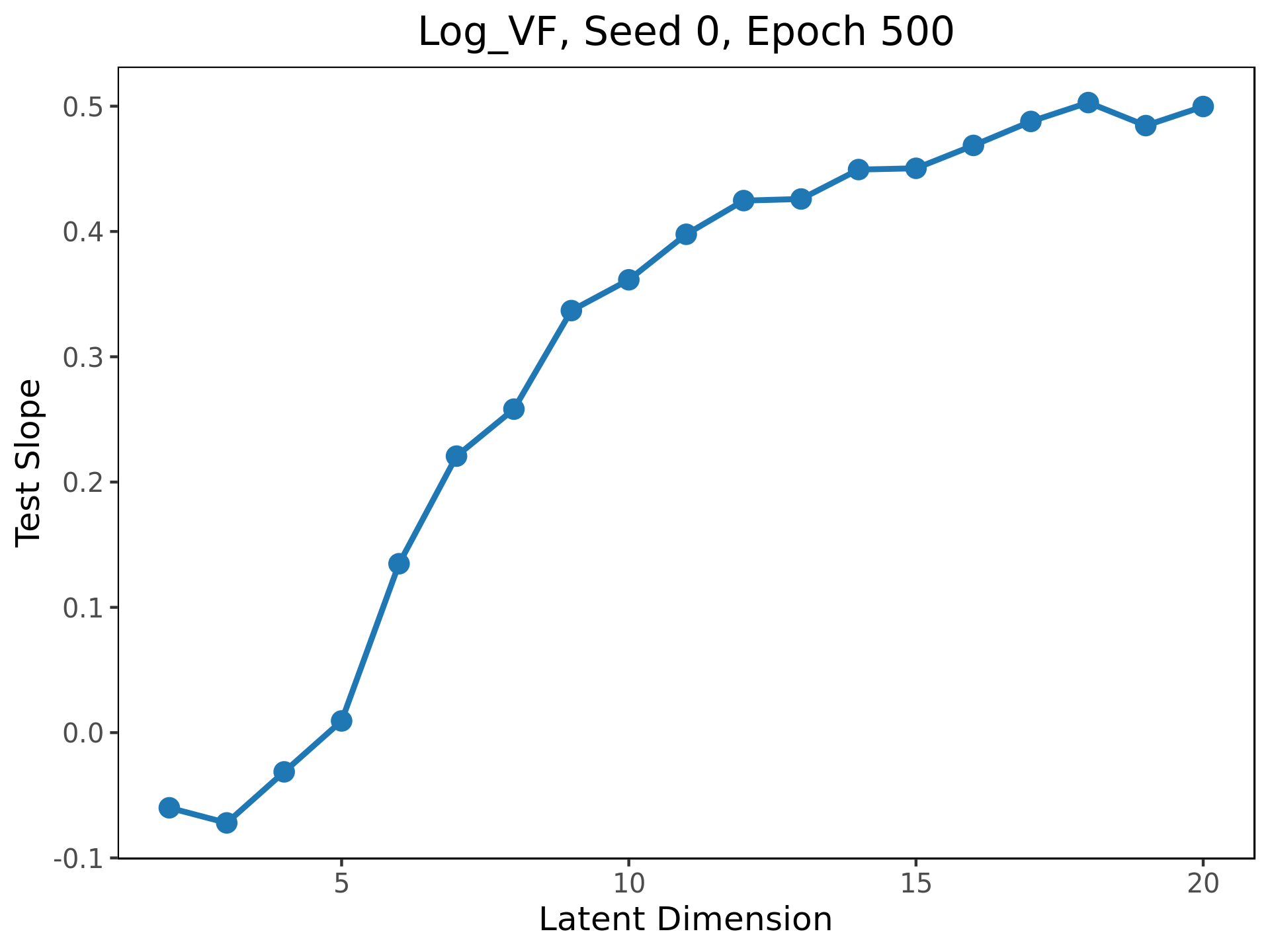}
\end{subfigure}
\begin{subfigure}[t]{0.3\textwidth}
\includegraphics[width=\textwidth]{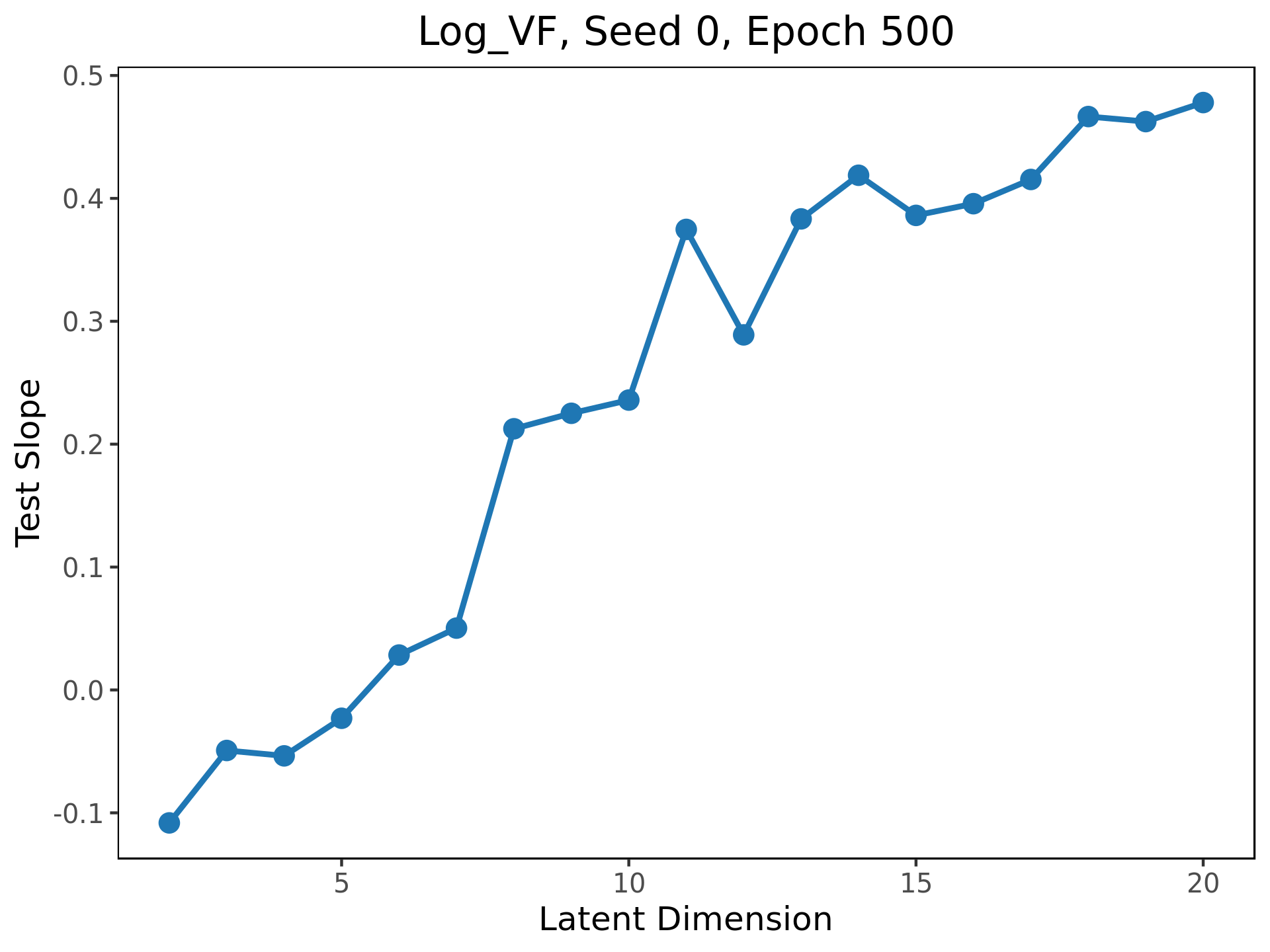}
\end{subfigure}
\newline 
\begin{subfigure}[t]{0.3\textwidth}
\includegraphics[width=\textwidth]{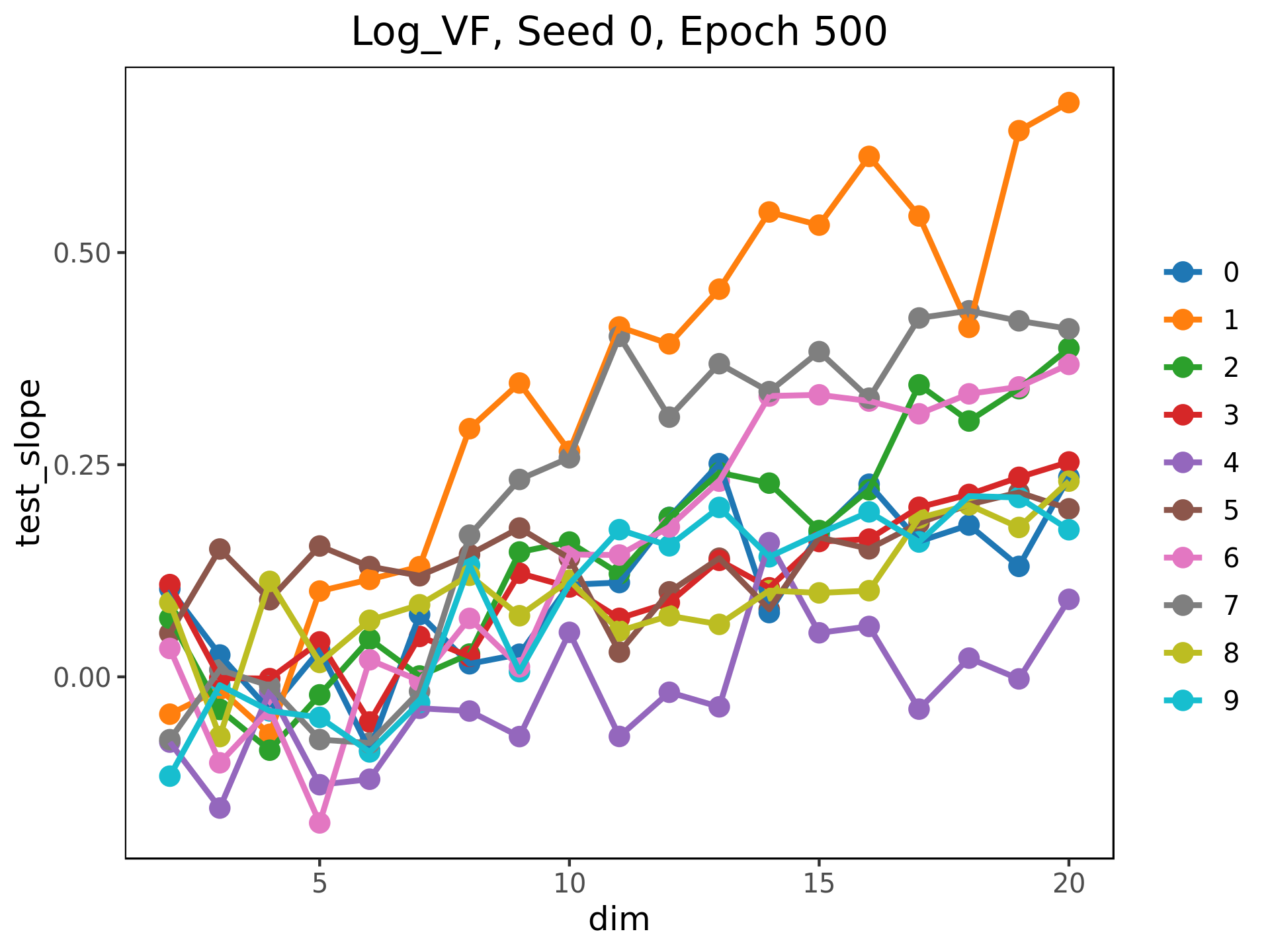}
\end{subfigure}
\begin{subfigure}[t]{0.3\textwidth}
\includegraphics[width=\textwidth]{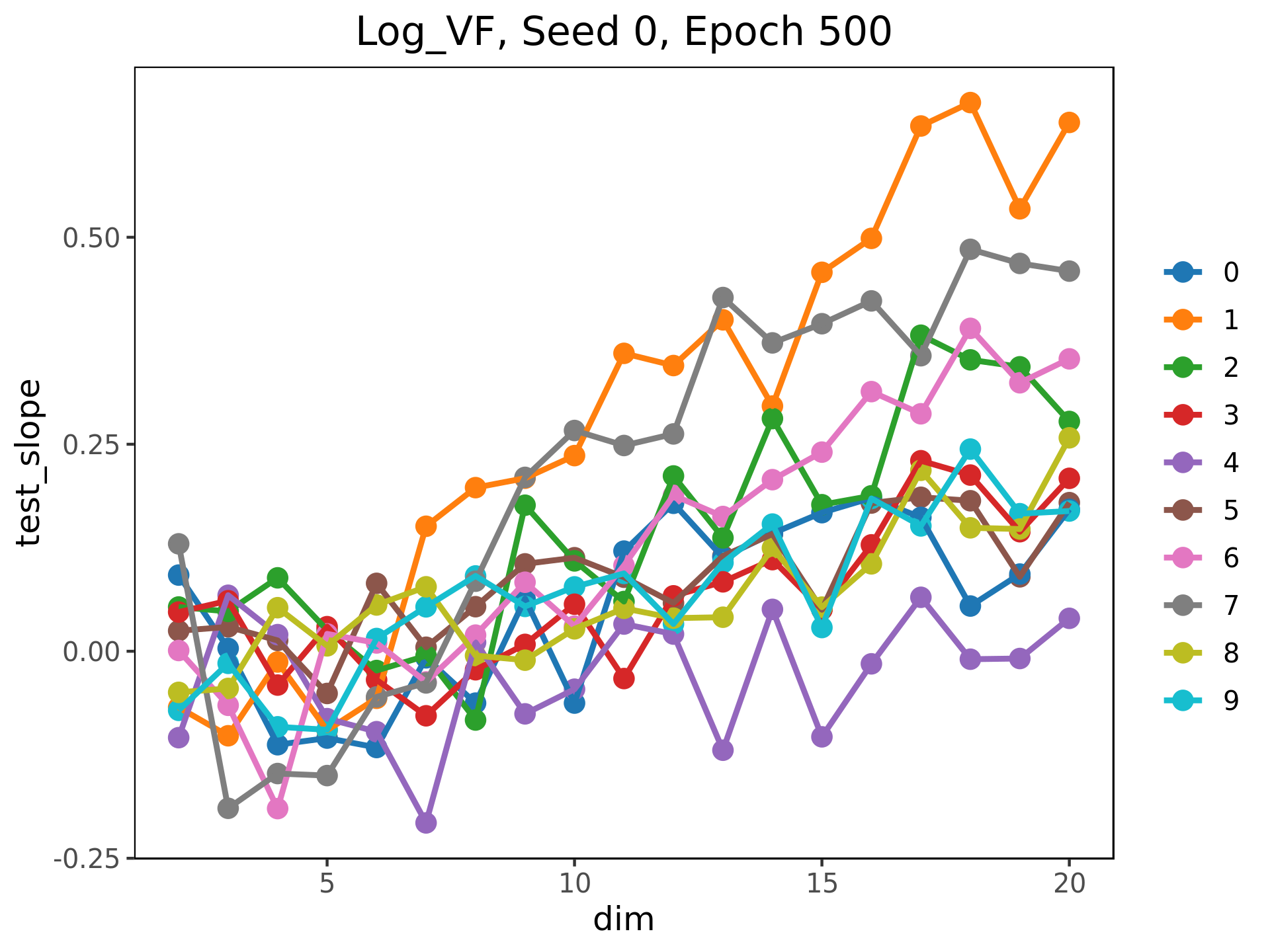}
\end{subfigure}
\begin{subfigure}[t]{0.3\textwidth}
\includegraphics[width=\textwidth]{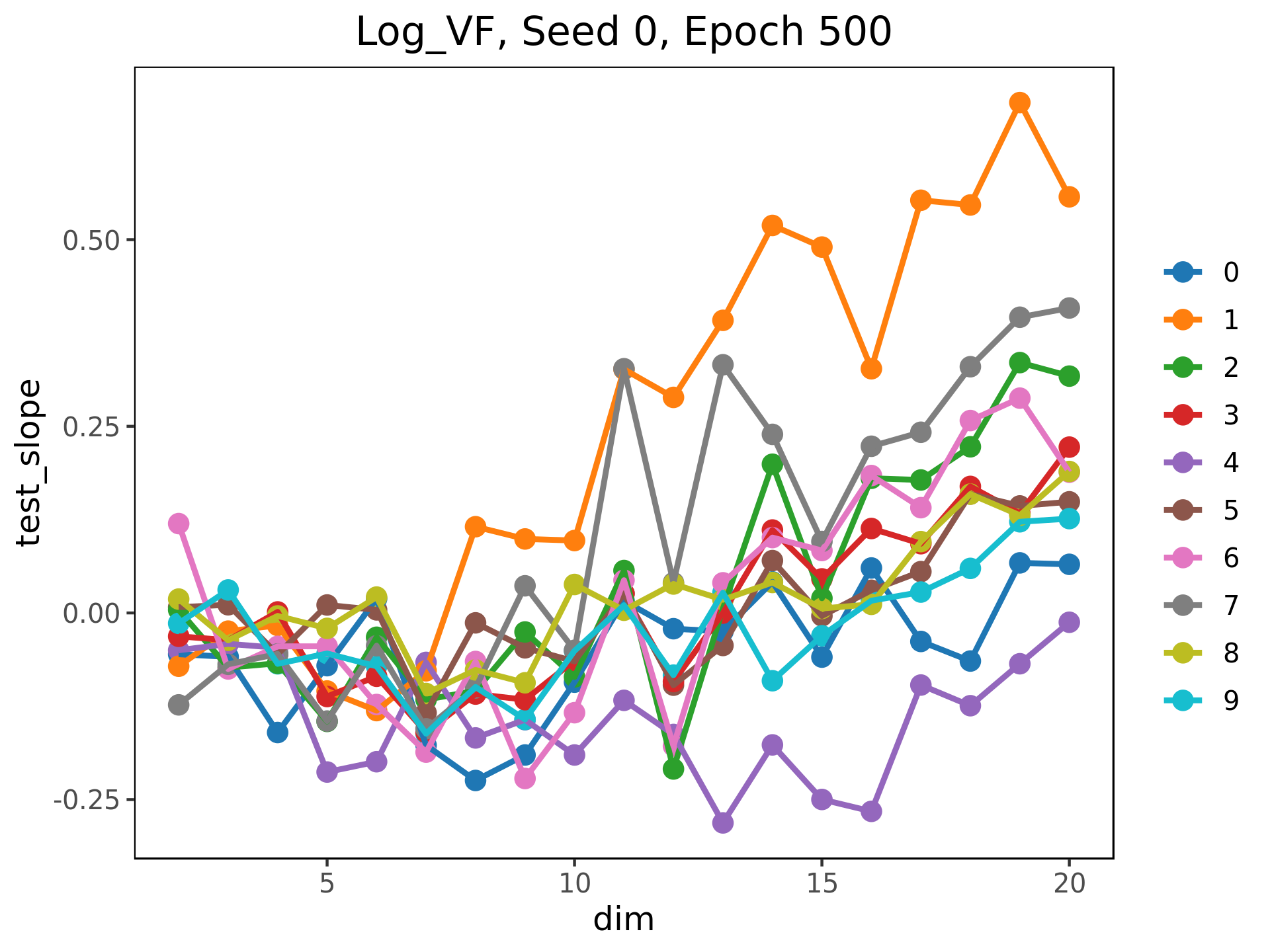}
\end{subfigure}
\caption{The expected increase (in units of standard deviation) of MSE for a test point given a standard deviation increase in log volume form increases with latent dimension. This figure shows the growth for $J_{\cL, 5}$.}\label{fig:log_pred}
\end{figure} 

Finally, it is worth noting that computing $\ln(\omega_*)$ is computationally expensive, especially for $* = \cI$. A much simpler quantity is $\Tr(J_{\cL})$ and $\Tr(J_\cI)$. We also test the predictive value of the trace of the Jacobians on the reconstruction error by running the regression \ba MSE = \beta_0 + \beta_1 Trace + \beta_2 test\_indicator +  \beta_3 test\_indicator \times Trace \;.  \label{eq:traceregression}\ea The coefficients for this regressions in the four layer experiment is given in Table \ref{table:traceregression}. Figure \ref{fig:trace_pred} shows the slopes of this regression for each experiment across latent dimension. Note from Table \ref{table:traceregression} that the intercepts for these regressions remain near $0$. 

As with the regressions involving $\log(\omega_*)$, note that  as latent dimension increases, the coefficient $\beta_3$ is consistently positive. In other words, the value of the trace as a predictor of the reconstruct error on test points is similar to that of $\log(\omega_*)$. Both of these are good candidate predictors of whether or not a trained network will generalize on a new data point. Therefore, while the trace may be a less satisfactory candidate for geometrically understanding the distortion between the encoder and the decoder halves of an autoencoder, in practice, it seems to perform as well, with the added advantage of being simpler to compute.


\begin{table}[h]
\centering
\resizebox{\textwidth}{!}{
\begin{filecontents*}{figs_and_tabs/graphs_seed-0_epochs-300/exp-default/coef_df_Trace_combined.csv}
dim,,,Small,,,,,Big,,
,Intercept,Trace,test_ind,test_Trace,test_slope,Intercept,Trace,test_ind,test_Trace,test_slope
2,-0.0024,-0.0018,0.016,-0.0526,-0.0544,-0.0024,0.096,0.017,0.0134,0.11
3,-0.0054,0.0038,0.038,-0.0369,-0.0331,-0.0054,0.083,0.037,0.1177,0.2
4,-0.0079,0.0151,0.055,-0.0096,0.0055,-0.0079,0.071,0.057,0.4135,0.48
5,-0.0143,0.0658,0.1,-0.034,0.0318,-0.0126,0.221,0.089,-0.0217,0.2
6,-0.0121,0.195,0.085,0.0112,0.2061,-0.0117,0.32,0.082,0.0427,0.36
7,-0.0129,0.2323,0.09,-0.0023,0.2299,-0.0122,0.351,0.086,0.0473,0.4
8,-0.0108,0.2753,0.075,-0.006,0.2693,-0.0091,0.369,0.063,0.0474,0.42
9,-0.0103,0.3465,0.073,0.0345,0.3809,-0.0093,0.435,0.065,0.0556,0.49
10,-0.0113,0.3638,0.08,0.016,0.3797,-0.0099,0.404,0.069,0.0366,0.44
11,-0.0095,0.3906,0.067,0.0241,0.4147,-0.0092,0.433,0.065,0.0287,0.46
12,-0.0059,0.4315,0.041,0.0026,0.4341,-0.0056,0.46,0.04,0.0045,0.46
13,-0.0096,0.4296,0.067,0.0076,0.4372,-0.0091,0.454,0.065,0.0228,0.48
14,-0.0085,0.4352,0.06,0.0227,0.4579,-0.0073,0.455,0.052,0.0223,0.48
15,-0.0087,0.4595,0.061,0.017,0.4765,-0.0085,0.483,0.06,0.0186,0.5
16,-0.0084,0.4722,0.06,0.0254,0.4975,-0.0084,0.489,0.059,0.0223,0.51
17,-0.0067,0.5127,0.048,0.0083,0.5209,-0.0064,0.513,0.045,0.0131,0.53
18,-0.0091,0.5361,0.065,0.022,0.5581,-0.0087,0.544,0.062,0.0177,0.56
19,-0.009,0.5141,0.064,0.016,0.5302,-0.0082,0.53,0.058,0.0147,0.54
20,-0.0078,0.5106,0.056,0.0239,0.5345,-0.0072,0.52,0.052,0.0274,0.55
\end{filecontents*}
\csvreader[
	tabular = |l||r|r|r|r|r||r|r|r|r|r|,
	table head = \hline \bfseries{dim} & \multicolumn{5}{|c||}{{\bf Coefficients for the latent space }$\cL$} & \multicolumn{5}{|c|}{{\bf Coefficients for the input space }$\cI$}\\\hline,
	late after last line=\\\hline
]{figs_and_tabs/graphs_seed-0_epochs-300/exp-default/coef_df_Trace_combined.csv}{}{\csvlinetotablerow}
}
\caption{Coefficients of the linear regression in equation \eqref{eq:traceregression} for $\Tr J_\cI$ and $\Tr J_\cL$ across latent dimensions. Run on data from seed 0, exp-default and 300 epochs.}
\label{table:traceregression}
\end{table}

\begin{figure}
\begin{subfigure}[t]{0.3\textwidth}
\includegraphics[width=\textwidth]{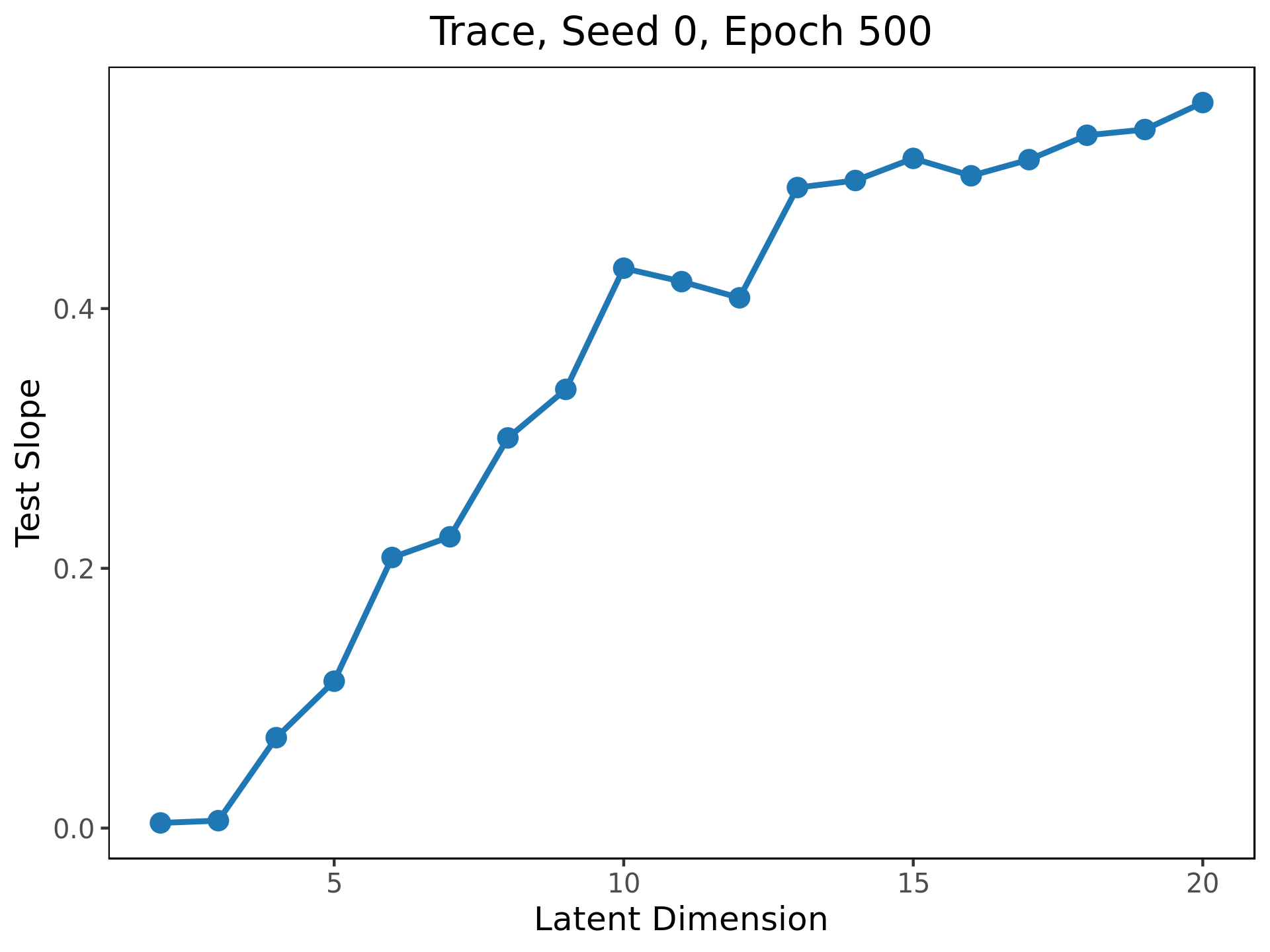}
\end{subfigure}
\begin{subfigure}[t]{0.3\textwidth}
\includegraphics[width=\textwidth]{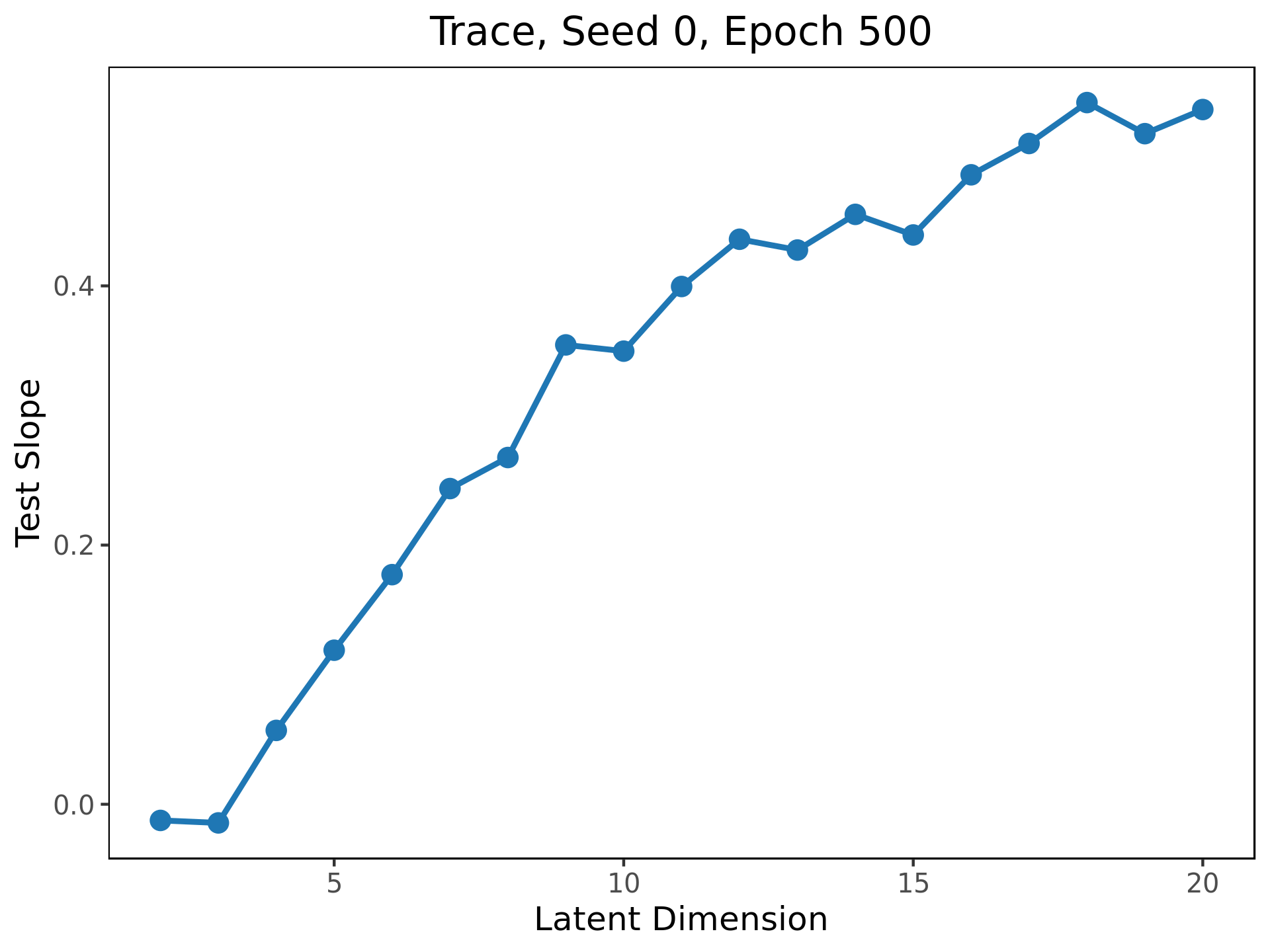}
\end{subfigure}
\begin{subfigure}[t]{0.3\textwidth}
\includegraphics[width=\textwidth]{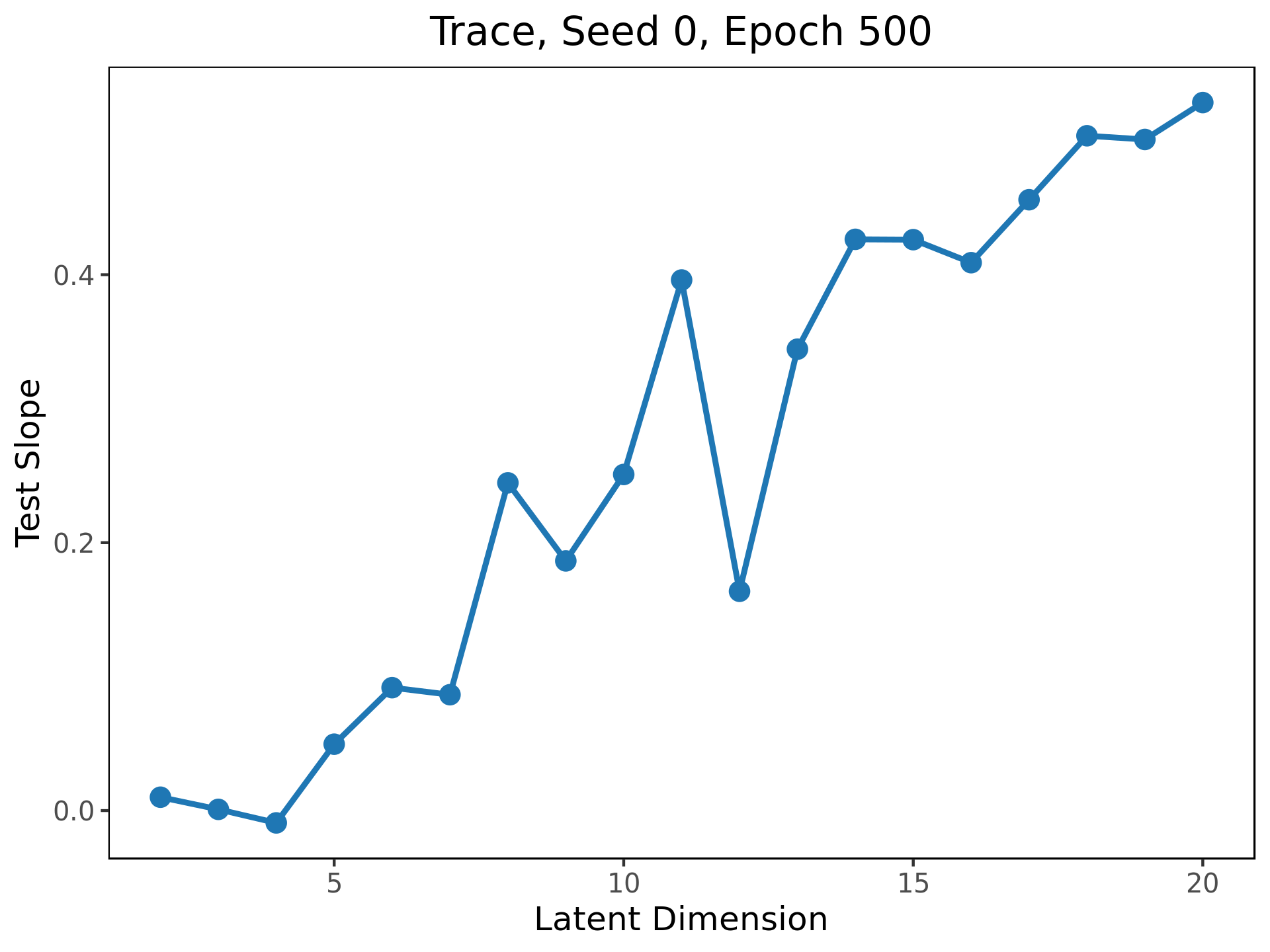}
\end{subfigure}
\newline 
\begin{subfigure}[t]{0.3\textwidth}
\includegraphics[width=\textwidth]{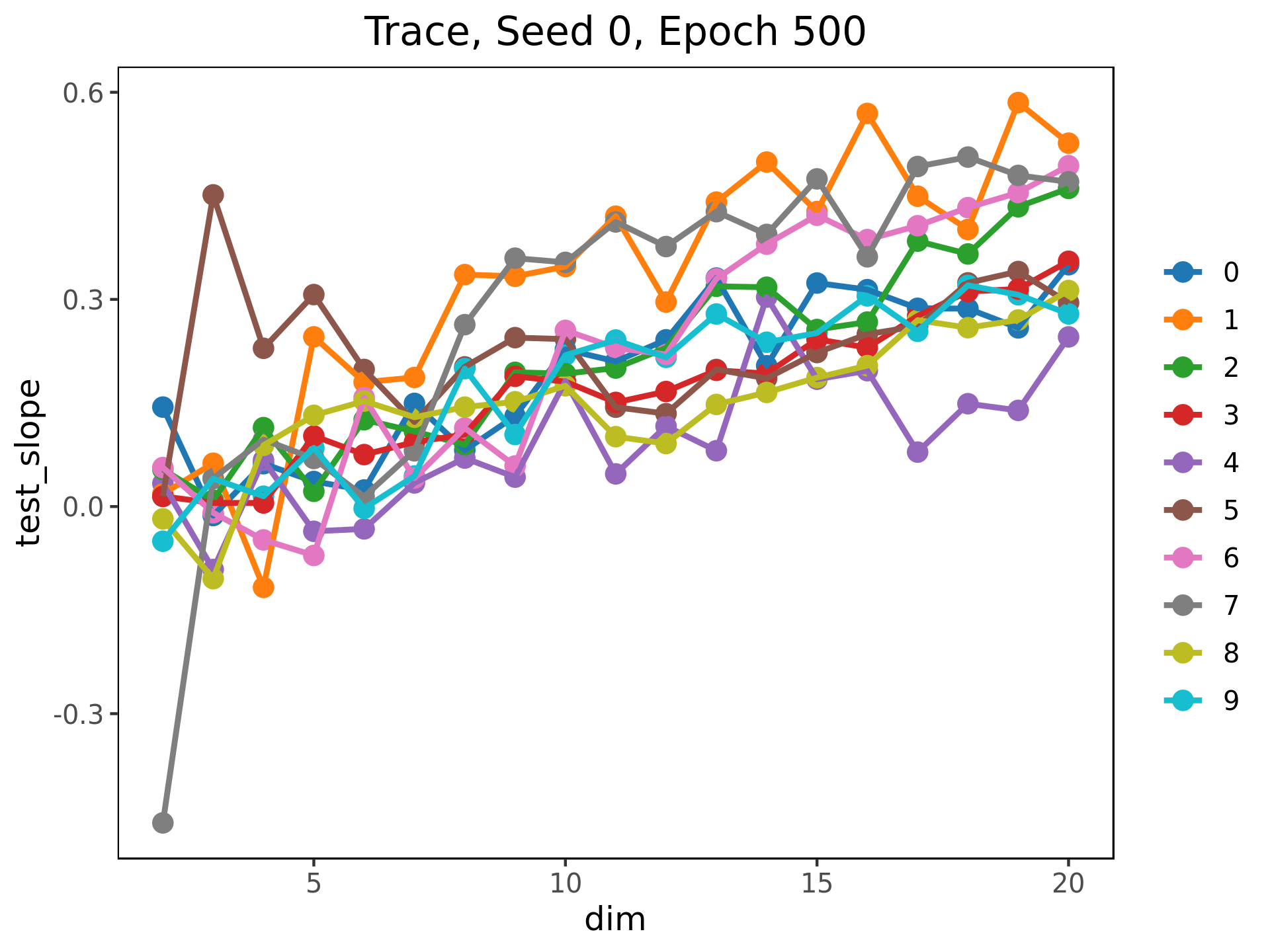}
\end{subfigure}
\begin{subfigure}[t]{0.3\textwidth}
\includegraphics[width=\textwidth]{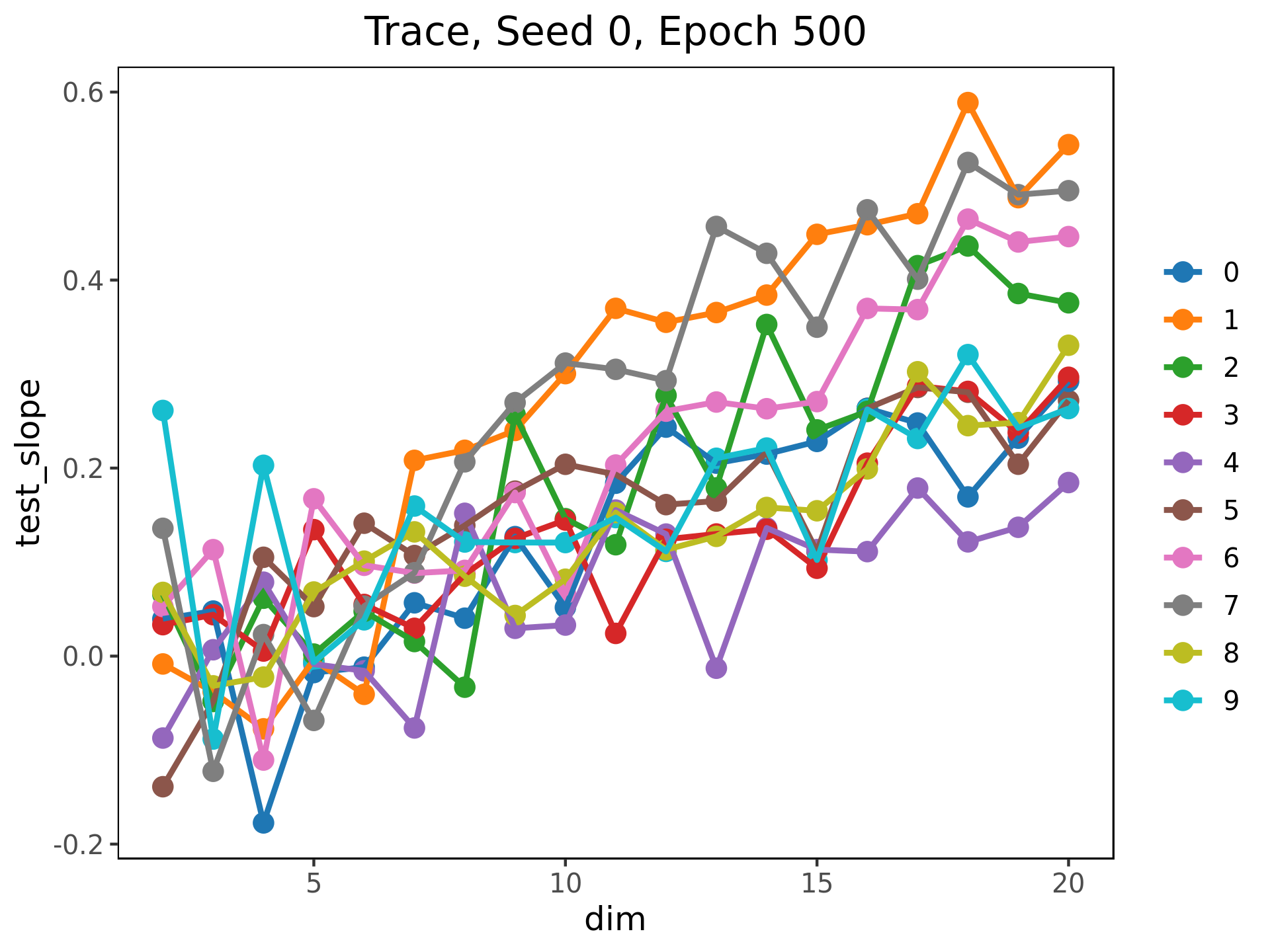}
\end{subfigure}
\begin{subfigure}[t]{0.3\textwidth}
\includegraphics[width=\textwidth]{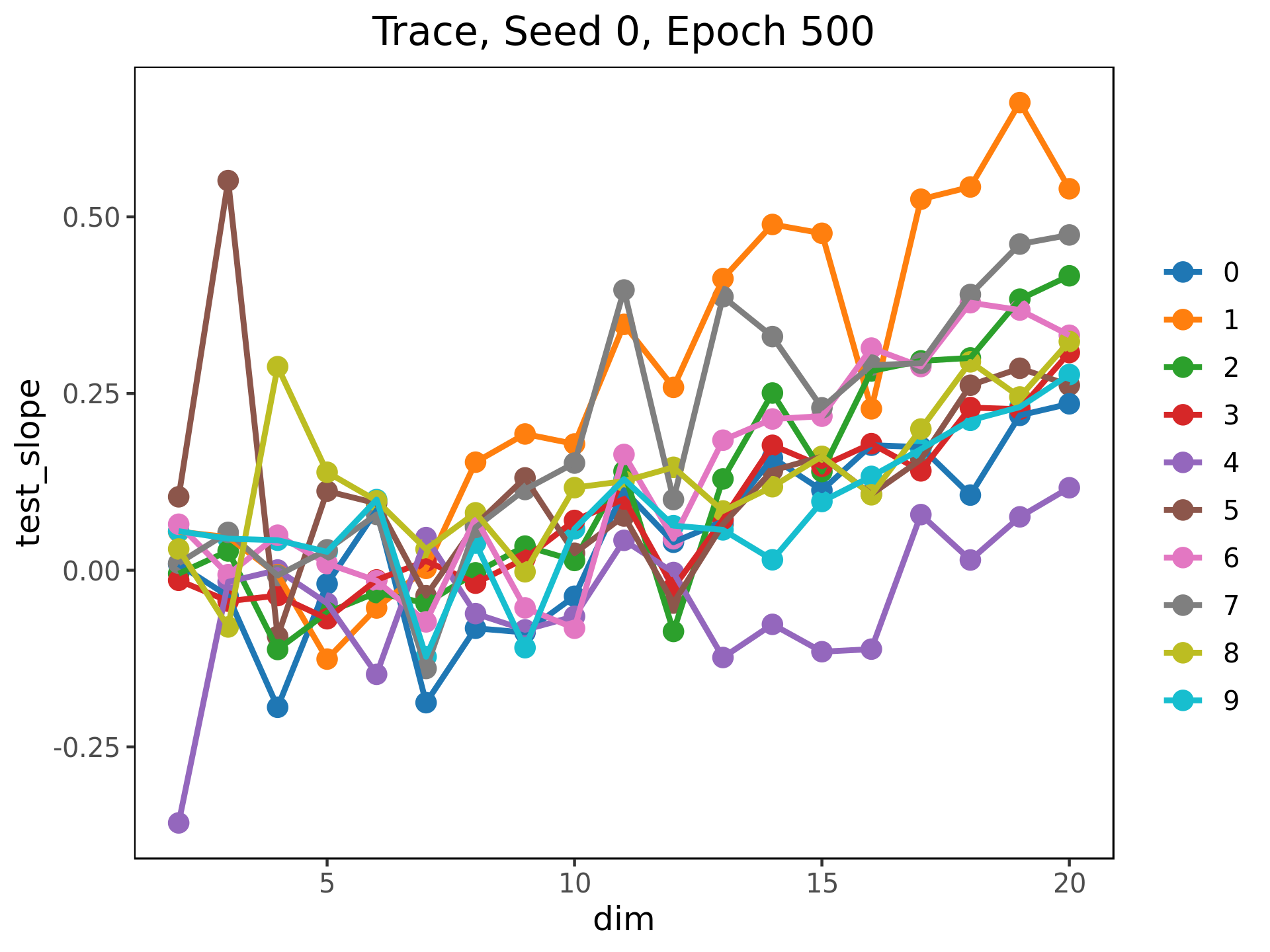}
\end{subfigure}
\caption{The expected increase (in units of standard deviation) of MSE for a test point given a standard deviation increase in trace increases with latent dimension.}\label{fig:trace_pred}
\end{figure}

\section{Conclusion and future work}
In this paper, we have examined the eigenvalues of two Jacobians associated to an autoencoder, $J_\cI$ and $J_\cL$. In doing so, we have seen that the eigenvalues are qualitatively consistent with a situation where the autoencoder is, to first order, locally a projection onto some data manifold $M_\cD$. Recall that the autoencoder is trained to be close to the identity on the set $\cD$ to $0^{th}$ order, where $\cD$ is conjectured to lie noisily around $M_\cD$. Therefore, it is noteworthy that a network trained to perform a zeroth order approximation also gives a good first order approximation. 

Furthermore, we note that two quantities derived from the eigenvalues of $J_{\cI}$ ($\Tr(J_\cL)$ and $\log(\omega_\cL)$) and $J_{\cI}$ ($\Tr(J_\cI)$ and $\log(\omega_\cI)$) are good predictors of where the autoencoder fails to generalize. Namely, all four of these quantities predict a higher reconstruction loss for test points than they do for training points. The beauty of using these quantities as predictors of a trained network's generalizability is that these methods do not require any knowledge of the dataset used for training and validation. In future work, we will apply the networks trained in for the experiments in this paper to related, but distinct datasets (for example, other datasets in the MNIST family) to see how well these candidate predictors perform on deployment data that is different from the training and validation data.

While the natural quantities to study are the eigenvalues of $J_\cI$, the structure of the autoencoder presents a problem in this case. Since the latent dimension is small ($\dim(\cL) <<\dim(\cI)$), the matrix $J_\cI$ does not have full rank. Therefore, in order to calculate $\omega_\cI$, one must first calculate the top $d = \dim(\cL)$ eigenvalues of $J_\cI$. However, we find that the eigenvalues of $J_\cI$ are close to the eigenvalues of $J_\cL$ (which is consistent with the autoencoder being close to an identity on $M_\cD$). Therefore, we propose using the quantity $\log(\omega_\cL)$, which measures the distortion of the latent space induced by the function $E_{model}\circ D_{model}$ as a predictor of the reconstruction error instead of $\log(\omega_\cI)$.

There is one unresolved curiosity of note in this work. We observe that the median eigenvalues of $J_\cI$ and $J_\cL$ are consistently less than 1. For more complicated architectures (for instance, in the $5$ layer autoencoders) we see that training for more epochs increases these magnitude of the median eigenvalues, but never exceeds 1. 

Furthermore, it is well known in the literature that large Frobenius norms \cite{weightdecay, spectralnorm1, Frobenius} of $J_\cI$ leads to a loss of generalizability of the network. However, the analysis in this paper indicates that this is not an effect that is driven purely by the very large norms. In fact, even when the median eigenvalues of $J_\cI$ are less than 1, increasing the quantity $\log(\omega_\cL)$ predicts a higher reconstruction loss on both the test and train points. 

This implies that the phenomenon being observed is more complicated than the ideal scenario laid out in Section \ref{sec:model}, as that discussion would imply that the predicted reconstruction loss would decrease as the eigenvalues approached $1$ from below, then increased again as they surpassed it. This needs further investigation.

We also note that in studying the eigenvalues of a trained autoencoder, we do not find zero eigenvalues appearing after the conjectured intrinsic dimensions of various classes (see for instance, \cite{ID}. If the autoencoder actually projected onto the data manifold precisely, we would expect to see more zero eigenvalues, because we would expect only $\dim(M_\cD)$ nonzero eigenvalues.  However, simple autoencoders, such as the ones trained here, are likely to also approximate some ``noise" directions in addition to the data directions.  This should result in more eigenvalues being drawn away from zero, as the autoencoder has learned more features than are present in the data manifold and will thus try to map onto a larger-dimensional space.  It is our hope that a method which identifies these noise directions separately, such as a variational autoencoder, will show eigenvalues more consistent with those of a projection.


\end{document}